\documentclass[12pt]{article}
\usepackage{fullpage}

\usepackage{amssymb}
\usepackage{amsmath}
\usepackage{algorithmic}
\usepackage{algorithm}
\usepackage{amsthm}
\usepackage{amsfonts}
\usepackage{comment}
\usepackage{graphicx}
\usepackage{hyperref}
\usepackage{color}
\usepackage{wasysym}

\usepackage{amsmath,amsfonts,amssymb}
\usepackage{algorithm,algorithmic}
\usepackage{multirow}
\usepackage{cellspace}
\usepackage{setspace}
\usepackage{bm}
\usepackage{bbm}
\usepackage{subfigure}
\usepackage{graphicx}
\usepackage{tikz,pgfplots}
\usepackage{framed}
\usepackage{amsthm}

\usepackage[round]{natbib}
\setcitestyle{authoryear,round,citesep={;},aysep={,},yysep={;}}

\usepackage{hyperref}
\hypersetup{colorlinks,
            linkcolor=blue,
            citecolor=blue,
            urlcolor=magenta,
            linktocpage,
            plainpages=false}

\usepackage[capitalise]{cleveref}

\usepackage{times}
\usepackage{graphicx} 
\usepackage{subfigure} 

\usepackage{natbib}

\usepackage{algorithm}
\usepackage{algorithmic}

\usepackage{hyperref}

\usepackage{geometry}                
\usepackage[parfill]{parskip}    
\usepackage{subfigure}
\usepackage{graphicx}
\usepackage{amssymb}
\usepackage{epstopdf}
\usepackage{algorithm}
\usepackage{algorithmic}
\newtheorem{theorem}{Theorem}[section]
\newtheorem{definition}{Definition}[section]
\newtheorem{lemma}{Lemma}[section]

\renewcommand{\le}{~\leq~}
\renewcommand{\ge}{~\geq~}

\def\reals{{\mathcal R}}

\newcommand{\K}{\mathcal{K}}
\newcommand{\R}{\mathcal{R}}

\newcommand{\ignore}[1]{}

\def\reals{{\mathbb R}}

\def\bold0{\mathbf{0}}

\newcommand\E{\mbox{\bf E}}

\include{header}  

\usepackage[utf8]{inputenc} 
\usepackage[T1]{fontenc}    
\usepackage{hyperref}       
\usepackage{url}            
\usepackage{booktabs}       
\usepackage{amsfonts}       
\usepackage{nicefrac}       
\usepackage{microtype}      

 \newcommand{\tg}{\tilde{g}}

\newcommand{\sumtt}{\sum_{t=0}^{T-1}}

\renewcommand{\O}{\mathcal{O}}
\newcommand{\tO}{\tilde{\mathcal{O}}}

\renewcommand{\le}{\leq}
\newcommand{\eref}{\eqref}
\newcommand{\rA}{\rm{(A)}}
\newcommand{\rB}{\rm{(B)}}
\newcommand{\rC}{\rm{(C)}}
\newcommand{\rD}{\rm{(D)}}

\usepackage{amsmath}
\DeclareMathOperator*{\argmin}{arg\,min}

\usepackage{algorithmic}
\newtheorem*{lemma*}{Lemma}

\title{Online Adaptive Methods, Universality and Acceleration}%

\author{%
Kfir Y. Levy\footnote{ETH Zurich; \texttt{yehuda.levy@inf.ethz.ch}.}
\and
Alp Yurtsever\footnote{Laboratory for Information and Inference Systems (LIONS), EPFL ; \texttt{alp.yurtsever@epfl.ch}.}
\and 
Volkan Cevher\footnote{Laboratory for Information and Inference Systems (LIONS), EPFL ; \texttt{volkan.cevher@epfl.ch}.}
}

\begin{document}

\maketitle

\begin{abstract}
  We present a novel method for convex unconstrained optimization that, \emph{without any modifications}, ensures: \textbf{(i)} accelerated convergence rate for smooth objectives, \textbf{(ii)} standard convergence rate in the general (non-smooth) setting, and \textbf{(iii)}   standard convergence rate in the stochastic optimization setting.  To the best of our knowledge, this is the first method that \emph{simultaneously} applies to all of the above settings.  
  
At the heart of our method is an adaptive learning rate rule that employs importance weights, in the spirit of adaptive online learning algorithms  \citep{duchi2011adaptive,levy2017online},  combined with an update  that linearly couples two sequences, in the spirit of~\citep{AllenOrecchia2017}.
An empirical examination of our method demonstrates its applicability to the above mentioned scenarios and corroborates our theoretical findings.
\end{abstract}


\section{Introduction}
The accelerated gradient method of \citet{nesterov1983method} is one of the cornerstones of modern optimization.
Due to its appeal as a computationally efficient and fast method, it has found use in numerous  applications including: 
imaging~\citep{chambolle2011first}, compressed sensing~\citep{foucart2013mathematical}, and deep learning~\citep{sutskever2013importance}, amongst other.

Despite these merits, accelerated methods are less prevalent in Machine Learning  due to two major issues: \textbf{(i)} acceleration is inappropriate for handling noisy feedback,  and  \textbf{(ii)} acceleration requires the knowledge of the objective's smoothness. While each of these issues was separately resolved  in \citep{lan2012optimal,hu2009accelerated,xiao2010dual}, and respectively in \citep{nesterov2015universal}; it was unknown whether there exists an accelerated method that addresses both issues. In this work we propose  such a method.

Concretely,   \citet{nesterov2015universal} devises a method that  obtains an accelerated convergence rate of  $\O(1/T^2)$ for smooth convex objectives, and a standard rate of $\O(1/\sqrt{T})$ for non-smooth convex objectives, over $T$ iterations. This is done without any prior knowledge of the smoothness parameter, and is therefore referred to as a \emph{universal}\footnote{Following Nesterov's paper \citep{nesterov2015universal}, we say that an algorithm is \emph{universal} if it does not require to know in advance whether the objective is smooth or not. Note that universality does not mean a parameter free algorithm. Specifically, Nesterov's universal methods~\citep{nesterov2015universal} as well as ours are not parameter free.}
method.
Nonetheless, this method uses a line search technique in every round, and is therefore inappropriate for handling noisy feedback.
On the other hand,    \citet{lan2012optimal},  \citet{hu2009accelerated},  and  \citet{xiao2010dual},
devise  accelerated methods that are able to handle noisy feedback and obtain a convergence rate of $\O({1}/{T^2}+ {\sigma}/{\sqrt{T}})$, where $\sigma$ is the variance of the gradients.
However, these methods are not universal since  they require the knowledge of both $\sigma$ and of the smoothness.

%
%
Conversely,  adaptive first order methods  are very popular in Machine Learning, 
with\linebreak AdaGrad,~\citep{duchi2011adaptive}, being the most prominent method among this class.
AdaGrad is an online learning algorithm which adapts its learning rate using the feedback (gradients) received through the optimization process, and is known to successfully handle  noisy feedback. 
This renders AdaGrad as the method of choice in various learning applications.
Note however, that AdaGrad  (probably) can  not ensure acceleration. 
Moreover, it was so far  unknown whether AdaGrad is  able to exploit smoothness  in order to converge faster.

In this work we investigate unconstrained convex optimization. We suggest \textmd{AcceleGrad} (Alg.~\ref{algorithm:UniAccel}), a novel \emph{universal} method which employs an accelerated-gradient-like update rule together with an adaptive learning rate \`a la AdaGrad. Our contributions,

\begin{itemize}
\item We show that \textmd{AcceleGrad} obtains an accelerated  rate of $\O(1/T^2)$ in the smooth case and $\tO(1/\sqrt{T})$ in the general case, without any prior information of the objective's smoothness.
\item We show that \emph{without any modifications}, \textmd{AcceleGrad} ensures a convergence rate of 
$\tO(1/\sqrt{T})$ in the general stochastic convex case.
\item We also present a new result regarding the AdaGrad algorithm. 
We show that in the case of stochastic optimization with a smooth expected loss, AdaGrad ensures an 
$\O(1/T+\sigma/\sqrt{T})$ convergence rate, where $\sigma$ is the variance of the gradients. AdaGrad does not require a knowledge of the smoothness, hence this result  establishes the universality of AdaGrad (though without acceleration).
\end{itemize}

On the technical side our algorithm emoploys three simultaneous mechanisms:
learning rate adaptation in conjunction with importance weighting, in the spirit of adaptive online learning algorithms \citep{duchi2011adaptive,levy2017online}, combined with
an update rule that linearly couples two  sequences, in the spirit of ~\citep{AllenOrecchia2017}.

This paper is organized as follows. In Section~\ref{sec:Settings} we present our setup and review relevant background.
Our results and analysis for the offline setting are presented in Section~\ref{sec:Offline}, and Section~\ref{sec:Stoch} presents our results for the stochastic setting. In Section~\ref{sec:Exps} we present our empirical study, and Section~\ref{sec:Conclusion} concludes.
\vspace{-10pt}
\paragraph{Related Work:}
In his pioneering work,  \citet{nesterov1983method}, establishes an accelerated rate  for smooth convex optimization. This was later generalized in, \citep{nesterov2003introductory, beck2009fast}, to allow for general metrics and line search.

In recent years there has been a renewed interest in accelerated methods, 
with efforts being made to understand acceleration as well as to extend it beyond the standard offline optimization setting.

An extension of acceleration to handle stochastic feedback was developed in, 
\citep{lan2012optimal,hu2009accelerated,xiao2010dual,cohen2018acceleration}.
Acceleration for modern variance reduction optimization methods is explored in, \citep{shalev2014accelerated,Allenzhu2017-katyusha}, and generic templates to accelerating variance reduction algorithms are developed in, \citep{lin2015universal,frostig2015regularizing}. 
\citet{scieur2016regularized}, derives a scheme that enables hindsight acceleration of non-accelerated methods. 
In \citep{yurtsever2015universal}, the authors devise a universal accelerated method for primal dual problems.
 And
the connection between acceleration and ODEs is investigated in, \citep{su14,wibisono2016variational,flammarion2015averaging,lessard2016analysis,aujol2017optimal,attouch2015fast}.
Universal accelerated schemes are explored in \cite{nesterov2015universal,lan2015bundle,neumaier2016osga}, yet these works do not apply to the stochastic setting.
Alternative accelerated methods and interpretations are explored in, \citep{arjevani2016lower,bubeck2015geometric,diakonikolas2017accelerated}.

Curiously,  \citet{AllenOrecchia2017}, interpret acceleration as a linear coupling between gradient descent and mirror descent, our work builds on their ideas.
Our method also relies on ideas from  \citep{levy2017online}, where universal (non-accelerated) procedures are derived through a conversion scheme of online learning algorithms.


\section{Setting and Preliminaries}\label{sec:Settings}

We discuss the optimization of a convex function $f: \reals^d \mapsto \reals$.  Our goal is to (approximately) solve the following unconstrained optimization problem,
 $$
 \min_{x\in \reals^d}f(x)~.
 $$
We focus  on first order methods, i.e., methods  that only require gradient information, and consider both
smooth and non-smooth objectives. The former is  defined below,
\begin{definition}[$\beta$-smoothness] A function $f:\reals^d\mapsto \reals$ is $\beta$-smooth if,
\begin{align*}
&f(y) \leq f(x) + \nabla f(x) \cdot (y-x) + \frac{\beta}{2}\|x - y\|^2 ;\quad \forall x,y \in \reals^d 
\end{align*}
\end{definition}
It is well known that with the knowledge of the smoothness parameter, $\beta$, one may obtain fast convergence rates by an appropriate adaptation of the update rule. 
In this work we do not assume any such knowledge; instead we assume to be given a  bound on the distance between some initial point, $x_0$, and a global minimizer of the objective.

This is formalized  as follows: we are given  a compact convex set $\K$ that contains a global minimum of $f$, i.e., $\exists z\in\K$ such that $z\in\argmin_{x\in\reals^d}f(x)$. 
Thus, for any initial point, $x_0\in\K$, its distance from the global optimum is bounded by the diameter of the set, $D:= \max_{x,y\in\K}\|x-y\|$.
Note that we allow to choose points outside $\K$.
We also assume that the objective $f$ is $G$-Lipschitz, which translates to a bound of $G$ on the magnitudes of the (sub)-gradients.

An access to the exact gradients of the objective is not always possible. And in many scenarios we may only  access  an oracle 
which  provides noisy and unbiased gradient estimates.  This \emph{Stochatic Optimization} setting is prevalent in Machine Learning, and we discuss it more formally in Section~\ref{sec:Stoch}.

\paragraph{The AdaGrad Algorithm:}
\begin{algorithm}[t]
\caption{Adaptive  Gradient Method ($\text{AdaGrad}$) }
\label{algorithm:AdaGrad}
\begin{algorithmic}
\STATE \textbf{Input}: \#Iterations $T$, $x_1\in \K$, set $\K$ 
\FOR{$t=1 \ldots T$ }
\STATE {Calculate:} $g_t= \nabla f(x_t)$, and update,  $\eta_t = D \left(2 \sum_{\tau=1}^t \|g_\tau\|^2 \right)^{-1/2}$
\STATE {Update:}
$$x_{t+1}= \Pi_{\K}\left( x_{t}-\eta_t {g}_{t}\right)$$
\ENDFOR
\STATE {Output:}
$\bar{x}_T~=~ \frac{1}{T}\sum_{t=1}^{T}  x_{t}$
\end{algorithmic}
\end{algorithm}
The adaptive method presented in this paper is inspired  by AdaGrad (Alg.~\ref{algorithm:AdaGrad}), a well known online optimization method which employs an adaptive learning rate.
The following theorem states $\text{AdaGrad}$'s guarantees\footnote{Actually AdaGrad is well known to ensure regret guarantees in the online setting. For concreteness, Thm.~\ref{theorem:AdaGrad} provides error guarantees in the offline setting.}
,~\citep{duchi2011adaptive},
\begin{theorem}\label{theorem:AdaGrad}
Let $\K$ be a convex set with diameter $D$. Let $f$ be a convex  function. Then Algorithm~\ref{algorithm:AdaGrad} guarantees the following error;
\begin{align*}
f(\bar{x}_T) - \min_{x\in\K} f(x) \le {\sqrt{2D^2\sum_{t=1}^T \|g_t\|^2}}/{T}~.
\end{align*}
\end{theorem}

\paragraph{Notation:}   Euclidean norm is denoted by $\| \cdot\|$. Given a compact convex set $\K$ we
denote by  $\Pi_\K(\cdot)$  the projection onto the $\K$, i.e.  $\forall x\in \reals^d$,\;
$
\Pi_\K(x) = \argmin_{y\in\K}\|y-x\|^2~.
$


\section{Offline Setting}
\label{sec:Offline}

This section discusses the offline optimization setting where we have an access to the exact gradients of the objective.
We present our  method  in Algorithm~\ref{algorithm:UniAccel}, and   substantiate its universality by providing $O(1/T^2)$ rate in the smooth case (Thm.~\ref{thm:Main}), and  a rate of $O(\sqrt{\log T/T})$ in the general convex case (Thm.~\ref{thm:MainNonSmooth}).
The analysis for the smooth case appears in Section~\ref{sec:AnalysisSmooth} and we defer the proof of the non-smooth case to the Appendix.
 
 \begin{algorithm}[t]
\caption{Accelerated Adaptive Gradient Method (\textmd{AcceleGrad}) }
\label{algorithm:UniAccel}
\begin{algorithmic}
\STATE \textbf{Input}: \#Iterations $T$, $x_0\in \K$, diameter $D$, weights $\{\alpha_t\}_{t\in[T]}$, learning rate $\{\eta_t\}_{t\in[T]}$
\STATE {Set}: $y_0 = z_0 = x_0 $
\FOR{$t=0 \ldots T$ }
\STATE {Set} $\tau_t = {1}/{\alpha_t}$ 
\STATE {Update:}
\begin{align*}
 x_{t+1} &=  \tau_t z_t + (1-\tau_t) y_t~ , \quad \text{and define } \;g_t := \nabla f(x_{t+1}) \\    
z_{t+1} &= \Pi_{\K}\left(z_t - \alpha_t  \eta_t g_t \right) \\
y_{t+1} & =  x_{t+1} -\eta_t g_t
\end{align*}
\ENDFOR
\STATE {Output:}
$\bar{y}_T\propto \sum_{t=0}^{T-1} \alpha_t y_{t+1}$
\end{algorithmic}
\end{algorithm}

 \textmd{AcceleGrad} is summarized in Algorithm~\ref{algorithm:UniAccel}.
 Inspired by, \citep{AllenOrecchia2017}, our method linearly couples between two sequences $\{z_t\}_t, \{y_t\}_t$ into a sequence $\{x_{t+1}\}_t$. 
 Using  the gradient 
 , $g_t=\nabla f(x_{t+1})$, these sequences are then updated with the same learning rate,  $\eta_t$, yet with different \emph{reference points} and \emph{gradient magnitudes}. Concretely,  $y_{t+1}$ takes a gradient step starting at $x_{t+1}$. Conversely, for 
 $z_{t+1}$ we scale the gradient  by a factor of $\alpha_t$ and then 
 take a projected gradient step starting at $z_{t}$.
 Our method finally outputs a weighted average of the $\{y_{t+1}\}_t$ sequence.
 
 Our algorithm coincides with the method of~\citep{AllenOrecchia2017} upon taking \linebreak$\eta_t = 1/\beta$ and outputting the last iterate, ${y}_T$, rather then a weighted average;
 yet this method is not universal.  Below we present our \emph{$\beta$-independent} choice of learning rate and weights,
 \begin{equation} \label{eq:WeightsLearningRate}
  \eta_t = \frac{2D}{\left(G^2+\sum_{\tau=0}^t  \alpha_\tau^2  \|g_\tau\|^2 \right)^{1/2}}
\qquad \&\qquad
\alpha_t=
\begin{cases}
	1 	&\quad \text{$0\leq t \leq 2$ } \\ 
	\frac{1}{4}(t+1)            	&\quad \text{$t\geq 3$}\\ 
\end{cases}
\end{equation}
The learning rate that we suggest adapts similarly to AdaGrad.  Differently from AdaGrad we  consider the importance weights, $\alpha_t$, inside the learning rate rule;  an idea that we borrow from  \citep{levy2017online}. The weights that we employ are increasing with $t$, which in turn emphasizes recent queries.

Next we state the guarantees of \textmd{AcceleGrad} for the smooth and non-smooth cases,
\begin{theorem}\label{thm:Main}
Assume that $f$ is convex and $\beta$-smooth. Let $\K$ be a convex set with bounded diameter $D$, and  assume there exists a global minimizer for $f$ in $\K$.
Then  Algorithm~\ref{algorithm:UniAccel} with weights and learning rate as in Equation~\eqref{eq:WeightsLearningRate} ensures,
 \begin{align*}
 f(  \bar{y}_T)  - \min_{x\in \reals^d}f(x)
&\leq 
\O\left( \frac{DG + \beta D^2\log(\beta D/G)}{T^2} \right)
\end{align*} 
\end{theorem}
\textbf{Remark:}
Actually, in the smooth case we do not need a bound on the Lipschitz continuity, i.e., $G$ is only required in case that the objective is non-smooth. Concretely, if we know that $f$ is smooth then we may use  $\eta_t = {2D}{\left(\sum_{\tau=0}^t  \alpha_\tau^2  \|g_\tau\|^2 \right)^{-1/2}}$, which yields a rate of
$\O\left( \frac{\beta D^2\log(\beta D/\|g_0\|)}{T^2} \right)$.

Next we show that the exactly same algorithm provides guarantees in the general convex case
(see proof in Appendix~\ref{app:MainNonSmooth}),
\begin{theorem}\label{thm:MainNonSmooth}
Assume that $f$ is convex and $G$-Lipschitz. Let $\K$ be a convex set with bounded diameter $D$, and  assume there exists a global minimizer for $f$ in $\K$.
Then  Algorithm~\ref{algorithm:UniAccel} with weights and learning rate as in Equation~\eqref{eq:WeightsLearningRate} ensures,
  \begin{align*}
 f(  \bar{y}_T)  - \min_{x\in \reals^d}f(x)
&\leq 
\O\left( {GD}\sqrt{\log T}/{\sqrt{T}} 
 \right)
\end{align*} 
\end{theorem}
\textbf{Remark:}
For non-smooth objectives, we can  modify \textmd{AcceleGrad} and provide guarantees for the \emph{constrained} setting.
Concretely, using Alg.~\ref{algorithm:UniAccel} with a projection step for the $y_t$'s, i.e., \linebreak  $y_{t+1} = \Pi_\K(x_{t+1}-\eta_t g_t)$, then we can bound its error by $ f(  \bar{y}_T)  - \min_{x\in \K}f(x)\leq\O\left( {GD}\sqrt{\log T}/{\sqrt{T}}\right)$. This holds even in the case where minimizer over $\K$ is not a global one.

%

\subsection{Analysis of the Smooth Case} \label{sec:AnalysisSmooth}
Here we provide a proof sketch for Theorem~\ref{thm:Main}
(the full proof is deferred to Appendix~\ref{app:Main}) .
For brevity, we will use $z\in\K$ to denote a \emph{global mimimizer} of $f$ which belongs to $\K$.

Recall that  Algorithm~\ref{algorithm:UniAccel} outputs a weighted average of the queries. Consequently,  we may  employ Jensen's inequality to bound its error as follow,
\begin{align} \label{eq:JensenProofSketch}
f(\bar{y}_T) - f(z)
&\leq 
\frac{1}{\sumtt \alpha_t}\sum_{t=0}^{T-1} {\alpha_t}\left(  f(y_{t+1})-  f(z)\right) ~.
\end{align} 
Combining this with  $\sumtt \alpha_t \geq \Omega(T^2)$, implies that  in order to substantiate the proof it is sufficient to show that, $\sum_{t=0}^{T-1} {\alpha_t}\left(  f(y_{t+1})-  f(z)\right)$, is bounded by a constant. This is the bulk of the analysis.

We start with the following lemma which provides us with a bound on $\alpha_t \left(  f(y_{t+1})-f(z)\right) $,
\begin{lemma}\label{lem:LemGen}
Assume that $f$ is convex and $\beta$-smooth. Then for any
 sequence of non-negative weights $\{\alpha_t\}_{t\geq0}$, and learning rates $\{\eta_t\}_{t\geq0}$, Algorithm~\ref{algorithm:UniAccel} ensures the following to hold,
\begin{align*}
\alpha_t (f(y_{t+1})- f(z)) 
& \leq
 (\alpha_t^2-\alpha_t) ( f(y_{t}) - f(y_{t+1}) )
 + \frac{\alpha_t^2}{2}\left( \beta - \frac{1}{\eta_t} \right)\| y_{t+1} - x_{t+1}\|^2 \\
&\quad    + \frac{1}{2\eta_t}\left( \|z_t-z\|^2 -\|z_{t+1}-z\|^2 \right)   
\end{align*} 
\end{lemma}
Interestingly,  choosing $\eta_t \leq 1/\beta$, implies that the above term, $\frac{\alpha_t^2}{2}\left( \beta - \frac{1}{\eta_t} \right)\| y_{t+1} - x_{t+1}\|^2$, does not contribute to the sum. We can show that this choice facilitates a concise analysis establishing an error of $\O(\beta D^2/T^2)$ for $\bar{y}_T$\footnote{While we do not spell out this analysis, it is a simplified version of our proof for Thm.~\ref{thm:Main}.}.

Note however that  our learning rate does not depend on $\beta$, and therefore the mentioned term is not necessarily negative. This issue is one of the main challenges in our analysis.
Next we provide a proof sketch of Theorem~\ref{thm:Main}. The full proof is deferred to Appendix~\ref{app:Main}.
\begin{proof}[Proof Sketch of  Theorem~\ref{thm:Main}]
Lemma~\ref{lem:LemGen} enables to decompose $\sum_{t=0}^{T-1}\alpha_t (f(y_{t+1})- f(z))$,
\begin{align} \label{eq:LemmaBoundSketch}
\sum_{t=0}^{T-1}\alpha_t (f(y_{t+1})- f(z)) 
& \leq
\underset{\rA}{\underbrace{ \sum_{t=0}^{T-1}\frac{1}{2\eta_t}\left( \|z_t-z\|^2 -\|z_{t+1}-z\|^2 \right)   }} \nonumber\\
&\quad +
 \underset{\rB}{\underbrace{ \sum_{t=0}^{T-1}(\alpha_t^2-\alpha_t) ( f(y_{t}) - f(y_{t+1}) ) }}
 + 
  \underset{\rC}{\underbrace{  \sum_{t=0}^{T-1}\frac{\alpha_t^2}{2}\left( \beta - \frac{1}{\eta_t} \right)\| y_{t+1} - x_{t+1}\|^2  }}
 \end{align} 
Next we separately bound each of the above terms.
  
 \paragraph{(a) Bounding  $\rA$ :}
 Using the fact that $\{ 1/\eta_t\}_{t\in[T]}$ is monotonically increasing allows to  show,
\begin{align}\label{eq:etaSmoothSketch}
\sum_{t=0}^{T-1} \frac{1}{2\eta_t}\left( \|z_t-z\|^2 -\|z_{t+1}-z\|^2 \right)  
~\leq~
\frac{1}{2}\sum_{t=1}^{T-1} \|z_t-z\|^2\left(\frac{1}{\eta_t} -\frac{1}{\eta_{t-1}}  \right) +\frac{\|z_0-z\|^2}{2\eta_0} 
~\leq~
\frac{D^2}{2\eta_{T-1}}
\end{align}
where we used  $\|z_t-z\|\leq D$. 

\paragraph{(b) Bounding  $\rB$ :}
We will require the following property of the weights that we choose (Eq.~\eqref{eq:WeightsLearningRate}),
\begin{align}\label{eq:WeightsSketch}
(\alpha_t^2-\alpha_t) - (\alpha_{t-1}^2-\alpha_{t-1})\leq \alpha_{t-1}/2
\end{align}
 
 Now recall that $z:=\argmin_{x\in\reals^d}f(x)$, and 
let us denote the sub-optimality of $y_t$ by $\delta_t$, i.e. $\delta_t = f(y_t) - f(z)$. Noting that $\delta_t\geq 0$ we may show the following,
 \begin{align} \label{eq:xDiffSketch}
\sum_{t=0}^{T-1} (\alpha_t^2-\alpha_t)\left( f(y_{t}) - f(y_{t+1})  \right)  
&~=~
\sum_{t=0}^{T-1} (\alpha_t^2-\alpha_t)\left( \delta_t - \delta_{t+1}  \right)  \nonumber\\
&~\leq~
\sum_{t=1}^{T-1} ( (\alpha_t^2-\alpha_t) - (\alpha_{t-1}^2-\alpha_{t-1}))  \delta_t  \nonumber\\
&~\leq~
\frac{1}{2}\sum_{t=0}^{T-1} \alpha_t \left(f(y_{t+1})-f(z) \right)
\end{align}
Where the last inequality uses Equation~\eqref{eq:WeightsSketch} (see full proof for the complete derivation). 
 \paragraph{(c) Bounding  $\rC$ :}
Let us denote
$
\tau_\star := \max\left\{t\in\{0,\ldots,T-1\}:  2\beta \geq 1/\eta_t \right\}~.
$
We may now split the term $\rC$ according to $\tau_\star$,
\begin{align}\label{eq:KeySmoothSketch}
\rC
&=
\sum_{t=0}^{\tau_\star}\frac{\alpha_t^2}{2}\left( \beta - \frac{1}{\eta_t} \right)\| y_{t+1} - x_{t+1}\|^2 
+\sum_{t=\tau_\star+1}^{T-1}\frac{\alpha_t^2}{2}\left( \beta - \frac{1}{\eta_t} \right)\| y_{t+1} - x_{t+1}\|^2  \nonumber \\
&\leq
\frac{\beta}{2}\sum_{t=0}^{\tau_\star} \alpha_t^2\| y_{t+1} - x_{t+1}\|^2
 -\frac{1}{4}\sum_{t=\tau_\star+1}^{T-1} \frac{\alpha_t^2}{\eta_t}\| y_{t+1} - x_{t+1}\|^2  \nonumber \\
 &=
 \frac{\beta}{2}\sum_{t=0}^{\tau_\star}\eta_t^2 \alpha_t^2\|g_t \|^2
 -\frac{1}{4}\sum_{t=\tau_\star+1}^{T-1} \eta_t \alpha_t^2\|g_t \|^2  
\end{align}
where in the second line we use $2\beta\leq \frac{1}{\eta_t}$ which holds for $t> \tau_\star$, implying that $\beta-\frac{1}{\eta_t}\leq -\frac{1}{2\eta_t}$;
in the last line we use $\|y_{t+1}-x_{t+1}\| = \eta_t\|g_t\|$.

\paragraph{Final Bound :}
Combining the bounds in Equations~\eqref{eq:etaSmoothSketch},\eqref{eq:xDiffSketch},\eqref{eq:KeySmoothSketch} into Eq.~\eqref{eq:LemmaBoundSketch},  and re-arranging gives,
 \begin{align}\label{eq:AlmostFinalSmoothSketch}
\frac{1}{2}\sum_{t=0}^{T-1}\alpha_t (f(y_{t+1})- f(z))  
& ~\leq~
\underset{(*)}{\underbrace{
 \frac{D^2}{2\eta_{T-1}} 
 -\frac{1}{4}\sum_{t=\tau_\star+1}^{T-1} \eta_t \alpha_t^2\|g_t \|^2 }}
 + 
\underset{(**)}{\underbrace{
 \frac{\beta}{2}\sum_{t=0}^{\tau_\star}\eta_t^2 \alpha_t^2\|g_t \|^2
 }}
 \end{align} 
We are now in the intricate part of the proof where we need to show that the above is bounded by a constant. As we show next this crucially depends on our choice of the learning rate. To simplify the proof sketch we assume to be using , $\eta_t = 2D\left( \sum_{\tau=0}^{t} \alpha_\tau^2\|g_\tau\|^2\right)^{-1/2}$, i.e. taking $G=0$ in the learning rate. We will require the following lemma before we go on,
 \begin{lemma*}
For any non-negative numbers $a_1,\ldots, a_n$ the following holds:
\begin{equation*}
\sqrt{\sum_{i=1}^n a_i} \leq \sum_{i=1}^n \frac{a_i}{\sqrt{\sum_{j=1}^i a_j}} \leq 2\sqrt{\sum_{i=1}^n a_i}~.
\end{equation*}
\end{lemma*}
Equipped with the above lemma and using $\eta_t$ explicitly enables to  bound $(*)$, 
\begin{align*}
 (*) &~= ~
 \frac{D}{4}\left(\sum_{t=0}^{T-1}  \alpha_t^2  \|g_t\|^2 \right)^{1/2} -\frac{D}{2}\sum_{t=\tau_\star+1}^{T-1}  \frac{ \alpha_t^2\|g_t\|^2}{\left(\sum_{\tau=0}^{t}  \alpha_\tau^2  \|g_\tau\|^2 \right)^{1/2}} \nonumber \\
 &~\leq~
 \frac{D}{4}\sumtt \frac{ \alpha_t^2\|g_t\|^2}{\left(\sum_{\tau=0}^{t}  \alpha_\tau^2  \|g_\tau\|^2 \right)^{1/2}}
-\frac{D}{2}\sum_{t=\tau_\star+1}^{T-1}  \frac{ \alpha_t^2\|g_t\|^2}{\left(\sum_{\tau=0}^{t}  \alpha_\tau^2  \|g_\tau\|^2 \right)^{1/2}} \nonumber \\
 &~\leq~
\frac{D}{4}\sum_{t=0}^{\tau_\star} \frac{ \alpha_t^2\|g_t\|^2}{\left(\sum_{\tau=0}^{t}  \alpha_\tau^2  \|g_\tau\|^2 \right)^{1/2}}\nonumber \\
&~\leq~
\frac{D}{2}\left(\sum_{\tau=0}^{\tau_\star}  \alpha_\tau^2  \|g_\tau\|^2 \right)^{1/2}\nonumber \\
&~=~ \frac{D^2}{\eta_{\tau_\star}}~\leq~ 2\beta D^2
 \end{align*}
 where in the last inequality we have used the definition of $\tau_\star$ which implies that $1/\eta_{\tau_\star}\leq 2\beta$.

Using similar argumentation  allows to bound the term $(**)$ by $\O(\beta D^2\log\left(\beta D/\|g_0\|\right))$. 
Plugging these bounds back into Eq.~\eqref{eq:AlmostFinalSmoothSketch} we get,
$$
\sum_{t=0}^{T-1}\alpha_t (f(y_{t+1})- f(z)) \leq \O(\beta D^2\log\left(\beta D/\|g_0\|\right))~.
$$
Combining this with Eq.~\eqref{eq:JensenProofSketch} and noting that $\sumtt \alpha_t \geq T^2/32$, concludes the proof.
\end{proof}

\newpage
\section{Stochastic Setting}
\label{sec:Stoch}
This section discusses the stochastic optimization setup which is prevalent in Machine Learning scenarios.
We formally describe this setup and prove that Algorithm~\ref{algorithm:UniAccel}, \emph{without any modification}, is ensured to converge in this setting (Thm.~\ref{thm:MainNonSmoothStoch}).
Conversely, the universal gradient methods presented in \citep{nesterov2015universal} rely on a line search procedure, which requires exact gradients and function values, and are therefore inappropriate for stochastic optimization. 

As a related result we show that the AdaGrad algorithm (Alg.~\ref{algorithm:AdaGrad}) is universal and is able to exploit small  variance in order to ensure fast rates in the case of stochastic optimization with smooth expected loss (Thm.~\ref{thm:AdaGradStoch}).
We emphasize that AdaGrad does not require the smoothness nor a bound on the  variance. Conversely, previous works with this type  of guarantees, \cite{xiao2010dual,lan2012optimal}, require the knowledge of \emph{both of these parameters}. 

\vspace{5pt}
\textbf{Setup:}
We consider the problem of minimizing a convex function $f:\reals^d \mapsto\reals$.
We assume that optimization lasts for $T$ rounds; on each round $t=1,\ldots,T$, we may query a point $x_t\in\reals^d$, and receive a \emph{feedback}.
After the last round,  we choose $\bar{x}_T\in\reals^d$, and our performance measure is  the expected excess loss, defined as,
$$\E[f(\bar{x}_T)]- \min_{x\in\reals^d}f(x)~.$$
Here we assume that our feedback is a first order noisy oracle such that upon   querying this oracle with a point $x$, we receive a bounded 
 and unbiased gradient estimate, $\tg$, such  
 \begin{align} \label{eq:NoisyOracle}
 \E[\tg \vert x] = \nabla f(x);\quad \&\quad  \|\tg\|\leq G
 \end{align}
We also assume  that  the internal coin tosses (randomizations) of the oracle are independent. It is well known that variants of Stochastic Gradient Descent (SGD) are ensured to output an estimate $\bar{x}_T$ such that the excess loss is bounded by $O(1/\sqrt{T})$ for the setups of  stochastic convex optimization,~\cite{nemirovskii1983problem}.
Similarly to the offline setting we assume to be given a set $\K$ with bounded diameter $D$, such that there exists a global optimum of  $f$ in $\K$. \\

The next theorem substantiates the guarantees of Algorithm~\ref{algorithm:UniAccel} in the stochastic case,

\begin{theorem}\label{thm:MainNonSmoothStoch}
Assume that $f$ is convex and $G$-Lipschitz. Let $\K$ be a convex set with bounded diameter $D$, and  assume there exists a global minimizer for $f$ in $\K$.
Assume that we invoke  Algorithm~\ref{algorithm:UniAccel} but provide it with noisy gradient estimates (see Eq.~\eqref{eq:NoisyOracle}) rather then the exact ones.
Then  Algorithm~\ref{algorithm:UniAccel} with weights and learning rate as in Equation~\eqref{eq:WeightsLearningRate} ensures,
  \begin{align*}
\E[ f(  \bar{y}_T)]  - \min_{x\in \reals^d}f(x)
&\leq 
\O\left( {GD}\sqrt{\log T}/{\sqrt{T}} 
 \right)
\end{align*} 
\end{theorem}
The analysis of Theorem~\ref{thm:MainNonSmoothStoch} goes along similar lines to the proof of its offline counterpart (i.e., Thm.~\ref{thm:MainNonSmooth}). 
The full proof is deferred to Appendix~\ref{app:ProofMainNonSmoothStoch}.

It is well known that AdaGrad  (Alg.~\ref{algorithm:AdaGrad}) enjoys the standard rate of $\O(GD/\sqrt{T})$ in the stochastic setting. The next lemma demonstrates that: \textbf{(i)}  AdaGrad is universal, and \textbf{(ii)}  AdaGrad implicitly make use of smoothness and small  variance in the stochastic setting.

 
\begin{theorem}\label{thm:AdaGradStoch}
Assume that $f$ is convex and $\beta$-smooth. Let $\K$ be a convex set with bounded diameter $D$, and  assume there exists a global minimizer for $f$ in $\K$.
Assume that we invoke AdaGrad (Alg.~\ref{algorithm:AdaGrad}) but provide it with noisy gradient estimates (see Eq.~\eqref{eq:NoisyOracle}) rather then the exact ones.
Then,
  \begin{align*}
\E[ f(  \bar{x}_T)]  - \min_{x\in \reals^d}f(x)
&\leq 
\O\left( \frac{\beta D^2}{T} +\frac{\sigma D}{\sqrt{T}}
 \right)
\end{align*} 
where $\sigma^2$  bounds  the variance of the gradients, i.e.,
$
\forall x\in \reals^d;\;\E\left[\|\tg - \nabla f(x)\|^2\vert x\right]\leq \sigma^2~.
$
\end{theorem}
Next we provide a proof  of the above theorem,
\begin{proof}[Proof  of Theorem~\ref{thm:AdaGradStoch}]
Lets us denote by $\tg_t$ the noisy gradients received by AdaGrad upon querying $x_{t}$. In this case, by applying the regret guarantees of AdaGrad, \cite{duchi2011adaptive}, in conjunction to standard online to batch conversion technique, \cite{cesa2004generalization}, implies,
\begin{align}\label{eq:regret}
\sum_{t=1}^T\E \left( f(x_t) - \min_{x\in\K} f(x)\right) \leq \E \sqrt{2D^2\sum_{t=1}^T\|\tg_t\|^2}
\end{align}
Now lets us denote by $g_t$ the exact gradient at $x_t$, and decompose: $\|\tg_t\|\leq \|g_t\|+ \|\tg_t-g_t\|$. This gives,
\begin{align*}
 \sqrt{\sum_{t=1}^T\|\tg_t\|^2} 
 ~\leq~
 \sqrt{2\sum_{t=1}^T\|g_t\|^2 + 2\sum_{t=1}^T\|\tg_t-g_t\|^2 } 
 ~\leq~
 \sqrt{2\sum_{t=1}^T\|g_t\|^2} +\sqrt{ 2\sum_{t=1}^T\|\tg_t-g_t\|^2 }~.
\end{align*}
where the first inequality uses $(a+b)^2\leq 2a^2+2b^2$, and the second inequality uses $(a+b)^{1/2}\leq a^{1/2}+b^{1/2}$ for non-negative $a,b\in \reals$.
Combining the above with Eq.~\eqref{eq:regret} and applying Jensen's inequality with respect to the function $H(u) = \sqrt{u}$, gives,
\begin{align}\label{eq:regret3}
\sum_{t=1}^T\E \left( f(x_t) - \min_{x\in\K} f(x)\right)& \leq  2\sqrt{D^2\sum_{t=1}^T\E\|g_t\|^2}+2\sqrt{D^2\sum_{t=1}^T\E\|\tg_t-g_t\|^2} \nonumber\\
&\leq
2\sqrt{2\beta D^2\sum_{t=1}^T\E \left( f(x_t) - \min_{x\in\K} f(x)\right)}+2\sqrt{\sigma^2 D^2T} 
\end{align}
 the last line uses the lemma below, which holds since we assume $\K$ contains a global minimum.
\begin{lemma} \label{lemma:GradIneqSmooth}
Let $F:\reals^d \mapsto \reals$ be a $\beta$-smooth function, and let $x^* =\argmin_{x\in \reals^d}F(x)$, then,
$$ \| \nabla F(x)\|^2 \le 2\beta \left( F(x) - F(x^*)\right), \quad \forall x\in \reals^d~.$$
\end{lemma} 
Eq.~\eqref{eq:regret} enables to show, $\sum_{t=1}^T\E \left( f(x_t) - \min_{x\in\K} f(x)\right)\leq 4\beta D^2 + 2\sigma D\sqrt{T}$. Combining this together with 
the definition of $\bar{x}_T$ and Jensen's inequality   concludes the proof.
\end{proof}


\section{Experiments}
\label{sec:Exps}
In this section we compare \textmd{AcceleGrad} against AdaGrad (Alg.~\ref{algorithm:AdaGrad}) and universal gradient methods \citep{nesterov2015universal}, focusing on the effect of tuning parameters and the level of adaptivity. 

We consider smooth ($p=2$) and non-smooth ($p=1$) regression problems of the form
 $$
 \min_{x\in \reals^d} F(x)~: =~\| Ax - b\|_p^p ~.
 $$
We synthetically generate matrix $A \in \reals^{n \times d}$ and a point of interest $x^\natural \in \reals^d$ randomly, with entries independently drawn from standard Gaussian distribution. Then, we generate $b = Ax^\natural + \omega$, with Gaussian noise, $w\sim \mathcal{N}(0,\sigma^2)$ and $\sigma^2 = 10^{-2}$. 
We fix $n = 2000$ and $d = 500$.

Figure~\ref{fig:l1l2min} presents the results for the offline optimization setting, where we provide the exact gradients of $F$. All methods are initialized at the origin, and we choose $\mathcal{K}$ as the $\ell_2$ norm ball of diameter $D$.

\begin{figure}[t!]
\includegraphics[width=\textwidth]{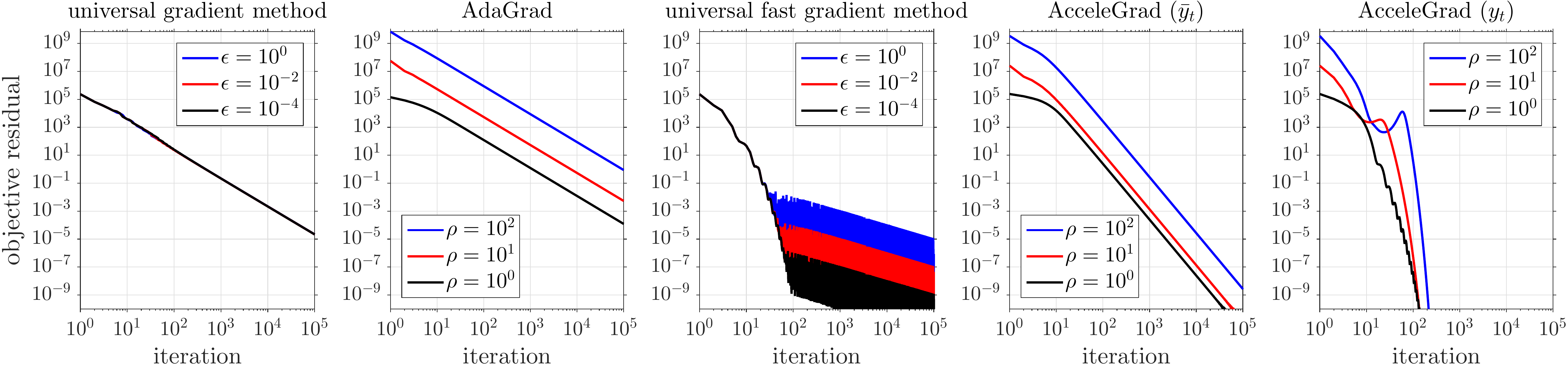} \\ 
\\
\vspace{-1.7em}
\includegraphics[width=\textwidth]{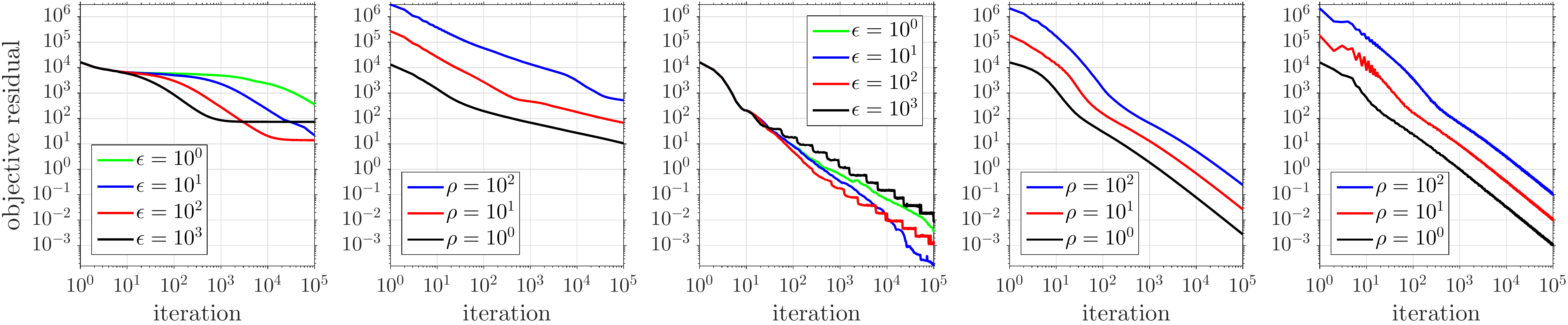}
\centering
\vspace{-3mm}
\caption{Comparison of universal methods at a smooth \textit{(top)} and a non-smooth \textit{(bottom)} problem.}
\label{fig:l1l2min}
\end{figure}

Universal gradient methods are based on an inexact line-search technique that requires an input parameter $\epsilon$. Moreover, these methods have convergence guarantees only up to $\frac{\epsilon}{2}$-suboptimality. 
For smooth problems, these methods perform better with smaller $\epsilon$. 
In stark contrast, for the non-smooth problems, small $\epsilon$ causes late adaptation, and large $\epsilon$ ends up with early saturation. Tuning is a major problem for these methods, since it requires rough knowledge of the optimal value.

Universal gradient method (also the fast version) provably requires two line-search iterations on average at each outer iteration. 
Consequently, it performs two data pass at each iteration (four for the fast version), while AdaGrad and \textmd{AcceleGrad} require only a single data pass. 

The parameter  $\rho$ denotes the ratio between $D/2$ and the distance between initial point and the solution. 
Parameter $D$ plays a major role on the step-size of AdaGrad and \textmd{AcceleGrad}.
Overestimating $D$ causes an overshoot in the first iterations. 
\textmd{AcceleGrad} consistently overperforms AdaGrad in the deterministic setting. 
As a final note, it needs to be mentioned that the iterates $y_t$ of AcceleGrad empirically converge faster than the averaged sequence $\bar{y}_T$. Note that for \textmd{AcceleGrad} we always take $G=0$, i.e., use $\eta_t = 2D\left(\sum_{\tau=0}^t \alpha_\tau^2 \|g_\tau\|^2\right)^{-1/2}$.

We also study the stochastic setup (Appendix~\ref{app:Numerics}), where we provide noisy gradients of $F$ based on minibatches.
As expected, universal line search methods \emph{fail} in this case, while \textmd{AcceleGrad} converges and performs similarly to AdaGrad.

\textbf{Large batches:} In Appendix~\ref{app:ExpLargeBatch} we show results on a real dataset which demonstrate the appeal of \textmd{AcceleGrad}  in the large-minibatch regime. We show that with the increase of batch size  the performance of \textmd{AcceleGrad} verses the number of gradient calculations does not degrade and might even \emph{improve}.
This is beneficial when we like to parallelize a stochastic optimization problem.
 Conversely, for AdaGrad we see  a clear degradation of the performance as we increase the batch size.

\vspace{-3pt}
\section{Conclusion and Future Work}
\label{sec:Conclusion}
\vspace{-2pt}
We have presented a novel universal method that may exploit smoothness in order to accelerate  while still being able to successfully handle noisy feedback.
Our current analysis only applies to unconstrained optimization problems.
Extending our work to the constrained setting is a natural future direction.
Another direction is to implicitly adapt the parameter $D$, this might be possible using ideas in the spirit of scale-free online algorithms, \cite{orabona2015scale,cutkosky2018black}.

\section*{Acknowledgement}
The authors would like to thank  Zalán Borsos for his insightful comments on the manuscript.

This project has received funding from the European Research Council (ERC) under the European Union's Horizon 2020 research and innovation programme (grant agreement no $725594$ - time-data). 
K.Y.L. is supported by the ETH Zurich Postdoctoral Fellowship and Marie Curie Actions for People COFUND program.

\bibliographystyle{abbrvnat}
\bibliography{bib}

\newpage

\appendix

\section{Proofs for the Smooth Case (Thm.~\ref{thm:Main})}
\label{app:Main}
Here we provide the complete proof  of Theorem~\ref{thm:Main}, and of the related lemmas.
For brevity, we will use $z\in\K$ to denote a \emph{global mimimizer} of $f$ which belongs to $\K$. 

Recall that  Algorithm~\ref{algorithm:UniAccel} outputs a weighted average of the queries. Consequently,  we may  employ Jensen's inequality to bound its error as follow,
\begin{align} \label{eq:JensenProof}
f(\bar{y}_T) - f(z)
&\leq 
\frac{1}{\sumtt \alpha_t}\sum_{t=0}^{T-1} {\alpha_t}\left(  f(y_{t+1})-  f(z)\right) ~.
\end{align} 
Combining this with  $\sumtt \alpha_t \geq \Omega(T^2)$, implies that  in order to substantiate the proof it is sufficient to show that, $\sum_{t=0}^{T-1} {\alpha_t}\left(  f(y_{t+1})-  f(z)\right)$, is bounded by a constant. This is the bulk of the analysis.

We start by recalling Lemma~\ref{lem:LemGen}  which provides us with abound on $\alpha_t \left(  f(y_{t+1})-f(z)\right) $,
\begin{lemma*}[Lemma~\ref{lem:LemGen}]
Assume that $f$ is convex and $\beta$-smooth. Then for any
 sequence of non-negative weights $\{\alpha_t\}_{t\geq0}$, and learning rates $\{\eta_t\}_{t\geq0}$, Algorithm~\ref{algorithm:UniAccel} ensures the following to hold,
\begin{align*}
\alpha_t (f(y_{t+1})- f(z)) 
& \leq
 (\alpha_t^2-\alpha_t) ( f(y_{t}) - f(y_{t+1}) )
 + \frac{\alpha_t^2}{2}\left( \beta - \frac{1}{\eta_t} \right)\| y_{t+1} - x_{t+1}\|^2 \\
&\quad    + \frac{1}{2\eta_t}\left( \|z_t-z\|^2 -\|z_{t+1}-z\|^2 \right)   
\end{align*} 
\end{lemma*}
The proof Lemma~\ref{lem:LemGen} appears in Appendix~\ref{app:Prooflem:LemGen}.
Next we prove Theorem~\ref{thm:Main}.
\newpage
\begin{proof}[Proof of Theorem~\ref{thm:Main}]
According to  Lemma~\ref{lem:LemGen},
\begin{align} \label{eq:LemmaBound}
\sum_{t=0}^{T-1}&\alpha_t (f(y_{t+1})- f(z))  \nonumber\\
& \leq
\underset{\rA}{\underbrace{ \sum_{t=0}^{T-1}\frac{1}{2\eta_t}\left( \|z_t-z\|^2 -\|z_{t+1}-z\|^2 \right)   }}
 +
 \underset{\rB}{\underbrace{ \sum_{t=0}^{T-1}(\alpha_t^2-\alpha_t) ( f(y_{t}) - f(y_{t+1}) ) }}
\nonumber\\
&\quad
 + 
  \underset{\rC}{\underbrace{  \sum_{t=0}^{T-1}\frac{\alpha_t^2}{2}\left( \beta - \frac{1}{\eta_t} \right)\| y_{t+1} - x_{t+1}\|^2  }}
 \end{align} 
It is natural to separately bound each of the sums above.
  
 \paragraph{(a) Bounding  $\rA$ :}
 Using the fact that $\{ 1/\eta_t\}_{t\in[T]}$ is monotonically increasing we may bound $\rA$  as follows,
\begin{align}\label{eq:etaSmooth}
\sum_{t=0}^{T-1} \frac{1}{2\eta_t}\left( \|z_t-z\|^2 -\|z_{t+1}-z\|^2 \right)  
&\leq
\frac{1}{2}\sum_{t=1}^{T-1} \|z_t-z\|^2\left(\frac{1}{\eta_t} -\frac{1}{\eta_{t-1}}  \right) +\frac{\|z_0-z\|^2}{2\eta_0} \nonumber \\
&\leq
\frac{D^2}{2\eta_{T-1}}
\end{align}
where we used  $\|z_t-z\|\leq D$.

 \paragraph{(b) Bounding  $\rB$ :}
  We will require the next lemma regarding the specific choice of the weights,
 \begin{lemma} \label{lem:Alphas1}
 The following holds for the $\alpha_t$'s which are described in  Eq.~\eqref{eq:WeightsLearningRate},
\begin{align*}
(\alpha_t^2-\alpha_t) - (\alpha_{t-1}^2-\alpha_{t-1})\leq \alpha_{t-1}/2
\end{align*}
\end{lemma}
Its proof appears in Appendix~\ref{app:Prooflem:Alphas1}.

We are now ready to bound $\rB$.
Recall that $z:=\argmin_{x\in\reals^d}f(x)$, and 
let us denote the sub-optimality of $y_t$ by $\delta_t$, i.e. $\delta_t = f(y_t) - f(z)$. Noting that $\delta_t\geq 0$ we may show the following,
 \begin{align} \label{eq:xDiff}
\sum_{t=0}^{T-1} (\alpha_t^2-\alpha_t)&\left( f(y_{t}) - f(y_{t+1})  \right)  \nonumber  \\
&=
\sum_{t=0}^{T-1} (\alpha_t^2-\alpha_t)\left( \delta_t - \delta_{t+1}  \right)   \nonumber  \\
&=
\sum_{t=1}^{T-1} ( (\alpha_t^2-\alpha_t) - (\alpha_{t-1}^2-\alpha_{t-1}))  \delta_t + (\alpha_0^2-\alpha_0) \delta_0 - (\alpha_{T-1}^2-\alpha_{T-1}) \delta_{T} \nonumber  \\
&\leq
\frac{1}{2}\sum_{t=1}^{T-1} \alpha_{t-1} \delta_t    \nonumber  \\
&\leq
\frac{1}{2}\sum_{t=1}^{T-1} \alpha_{t-1} \delta_t + \frac{1}{2}\alpha_{T-1}\delta_{T} \nonumber  \\
& =
\frac{1}{2}\sum_{t=0}^{T-1} \alpha_t \delta_{t+1} \nonumber  \\
&=
\frac{1}{2}\sum_{t=0}^{T-1} \alpha_t \left(f(y_{t+1})-f(z) \right)
 \end{align}
 where in the fourth line  we use  Lemma~\ref{lem:Alphas1}, we also use $\alpha_0^2-\alpha_0=0$ and 
 $\alpha_{T-1}^2-\alpha_{T-1}\geq 0$.

 \paragraph{(c) Bounding  $\rC$ :}
Let us denote $\tau_\star$ as follows:~
$
\tau_\star = \max\left\{t\in\{0,\ldots,T-1\}:  2\beta \geq 1/\eta_t \right\}~.
$
We may now split the last term as follows,
\begin{align}\label{eq:KeySmooth}
\sum_{t=0}^{T-1}&\frac{\alpha_t^2}{2}\left( \beta - \frac{1}{\eta_t} \right)\| y_{t+1} - x_{t+1}\|^2    \nonumber \\
&=
\sum_{t=0}^{\tau_\star}\frac{\alpha_t^2}{2}\left( \beta - \frac{1}{\eta_t} \right)\| y_{t+1} - x_{t+1}\|^2 
+\sum_{t=\tau_\star+1}^{T-1}\frac{\alpha_t^2}{2}\left( \beta - \frac{1}{\eta_t} \right)\| y_{t+1} - x_{t+1}\|^2  \nonumber \\
&\leq
\frac{\beta}{2}\sum_{t=0}^{\tau_\star} \alpha_t^2\| y_{t+1} - x_{t+1}\|^2
 -\frac{1}{4}\sum_{t=\tau_\star+1}^{T-1} \frac{\alpha_t^2}{\eta_t}\| y_{t+1} - x_{t+1}\|^2  \nonumber \\
 &=
 \frac{\beta}{2}\sum_{t=0}^{\tau_\star}\eta_t^2 \alpha_t^2\|g_t \|^2
 -\frac{1}{4}\sum_{t=\tau_\star+1}^{T-1} \eta_t \alpha_t^2\|g_t \|^2  
\end{align}
where in the third line we use $2\beta\leq \frac{1}{\eta_t}$ which holds for $t> \tau_\star$, implying that $\beta-\frac{1}{\eta_t}\leq -\frac{1}{2\eta_t}$;
in the fourth line we use $\|y_{t+1}-x_{t+1}\| = \eta_t\|g_t\|$.

\paragraph{Final Bound :}
Combining the bounds in Eq.~\eqref{eq:etaSmooth}-\eqref{eq:KeySmooth} into Eq.~\eqref{eq:LemmaBound}, we obtain,
\begin{align*}
\sum_{t=0}^{T-1}\alpha_t (f(y_{t+1})- f(z))  
& \leq
\frac{D^2}{2\eta_{T-1}}
 +
\frac{1}{2}\sum_{t=0}^{T-1} \alpha_t \left(f(y_{t+1})-f(z) \right)
\nonumber\\
&\quad
 + 
 \frac{\beta}{2}\sum_{t=0}^{\tau_\star}\eta_t^2 \alpha_t^2\|g_t \|^2 
 -\frac{1}{4}\sum_{t=\tau_\star+1}^{T-1} \eta_t \alpha_t^2\|g_t \|^2 
 \end{align*} 
 Re-arranging we get,
 \begin{align}\label{eq:AlmostFinalSmooth}
\frac{1}{2}\sum_{t=0}^{T-1}\alpha_t (f(y_{t+1})- f(z)) 
& \leq
\underset{(*)}{\underbrace{ \frac{D^2}{2\eta_{T-1}}
 -\frac{1}{4}\sum_{t=\tau_\star+1}^{T-1} \eta_t \alpha_t^2\|g_t \|^2  
  }}
 + 
 \underset{(**)}{\underbrace{
 \frac{\beta}{2}\sum_{t=0}^{\tau_\star}\eta_t^2 \alpha_t^2\|g_t \|^2 
 }}
 \end{align} 
 
 This is the intricate part of the proof where we show that the above is bounded by a constant. This crucially depends on our choice of the learning rate, i.e., $\eta_t = 2D\left(G^2+ \sum_{\tau=0}^{t} \alpha_\tau^2\|g_\tau\|^2\right)^{-1/2}$. We  require the following lemma (proof is found in Appendix~\ref{app:Prooflem:SqrtSumReversed}) before we go on,
 \begin{lemma}\label{lem:SqrtSumReversed}
For any non-negative numbers $a_1,\ldots, a_n$ the following holds:
\begin{equation*}
\sqrt{\sum_{i=1}^n a_i} \leq \sum_{i=1}^n \frac{a_i}{\sqrt{\sum_{j=1}^i a_j}} \leq 2\sqrt{\sum_{i=1}^n a_i}~.
\end{equation*}
\end{lemma}
Equipped with the above lemma and using $\eta_t$ explicitly enables to  bound $(*)$, 
\begin{align}\label{StarTermSmooth}
 (*) &~= ~
 \frac{D}{4}\left(G^2+\sum_{t=0}^{T-1}  \alpha_t^2  \|g_t\|^2 \right)^{1/2} -\frac{D}{2}\sum_{t=\tau_\star+1}^{T-1}  \frac{ \alpha_t^2\|g_t\|^2}{\left(G^2+\sum_{\tau=0}^{t}  \alpha_\tau^2  \|g_\tau\|^2 \right)^{1/2}} \nonumber \\
 &~\leq~
 \frac{D}{4}\left( 
 \frac{G^2}{(G^2)^{1/2}}+
 \sumtt \frac{ \alpha_t^2\|g_t\|^2}{\left(G^2+\sum_{\tau=0}^{t}  \alpha_\tau^2  \|g_\tau\|^2 \right)^{1/2}}
 \right)
-\frac{D}{2}\sum_{t=\tau_\star+1}^{T-1}  \frac{ \alpha_t^2\|g_t\|^2}{\left(G^2+\sum_{\tau=0}^{t}  \alpha_\tau^2  \|g_\tau\|^2 \right)^{1/2}} \nonumber \\
 &~\leq~
 \frac{DG}{4}+
\frac{D}{4}\sum_{t=0}^{\tau_\star} \frac{ \alpha_t^2\|g_t\|^2}{\left(G^2+\sum_{\tau=0}^{t}  \alpha_\tau^2  \|g_\tau\|^2 \right)^{1/2}}\nonumber \\
&~\leq~
\frac{DG}{4}+
\frac{D}{2}\left(\sum_{\tau=0}^{\tau_\star}  \alpha_\tau^2  \|g_\tau\|^2 \right)^{1/2}\nonumber \\
&~=~ 
\frac{DG}{4}+\frac{D^2}{\eta_{\tau_\star}} \nonumber \\
&~\leq~ 
{DG}/{4}+2\beta D^2
 \end{align}
where in the second line we use the left hand nequality of Lemma~\ref{lem:SqrtSumReversed};
 in the fourth line we use the right hand inequality of Lemma~\ref{lem:SqrtSumReversed}
 ; and  in the last line we have used the definition of $\tau_\star$ which implies that $1/\eta_{\tau_\star}\leq 2\beta$.

 We will also require the following lemma
 (proof is found in Appendix~\ref{app:Prooflem:Log_sum}),
  \begin{lemma} \label{lem:Log_sum}
For any non-negative real numbers $a_1,\ldots, a_n$,
\begin{align*}
\sum_{i=1}^n \frac{a_i}{1+\sum_{j=1}^i a_j} 
\le 
1+\log\left( 1+\sum_{i=1}^n a_i\right) ~.
\end{align*}
\end{lemma}
 Equipped with the above lemma and using $\eta_t$ explicitly enables to  bound $(**)$, 
 \begin{align}\label{eq:SmoothStar2}
  \frac{\beta}{2}\sum_{t=0}^{\tau_\star}\eta_t^2 \alpha_t^2\|g_t \|^2  
  &=
  \frac{4\beta D^2}{2}\sum_{t=0}^{\tau_\star} \frac{\alpha_t^2\|g_t \|^2}{G^2+\sum_{\tau=0}^t  \alpha_\tau^2  \|g_\tau\|^2}  \nonumber\\
  &=
 {2\beta D^2}\sum_{t=0}^{\tau_\star} \frac{\alpha_t^2(\|g_t \|/G)^2}{1+\sum_{\tau=0}^t  \alpha_\tau^2  (\|g_\tau\|/G)^2} \nonumber\\
 &\leq 
 {2\beta D^2}
 \left( 
 1+ \log\left((G/G)^2+\sum_{\tau=0}^{\tau_\star}  \alpha_\tau^2  (\|g_\tau\|/G)^2 \right) 
 \right)\nonumber\\
 &= 
 {2\beta D^2}
  \left(
   1+ \log\left(\frac{4D^2/G^2}{\eta_{\tau_\star}^2}\right)
   \right) \nonumber\\
  &\leq
 {2\beta D^2}
  \left(
   1+ 2\log\left({4\beta D/G}\right)
   \right) 
 \end{align}
 where in the third line we
used Lemma~\ref{lem:Log_sum}, and
 in the last line we have used the definition of $\tau_\star$ which implies that $1/\eta_{\tau_\star}\leq 2\beta$. 
 Combining Equations~\eqref{StarTermSmooth}, \eqref{eq:SmoothStar2} back into Eq.~\eqref{eq:AlmostFinalSmooth} and 
using Jensen's inequality we are now ready to establish the final bound,
\begin{align*}
f(\bar{y}_T) - f(z) 
&\leq 
\frac{\sum_{t=0}^{T-1}\alpha_t (f(y_{t+1}) - f(z))}{\sumtt \alpha_t} \\
&\leq
\frac{{DG}/{2}
 +
8\beta D^2
  \left(
   1+ \log\left({4\beta D}/{G}\right)
   \right) 
 }{T^2/32} \\
&=
O\left(\frac{DG + \beta D^2\log(\beta D/G)}{T^2}\right)~.
\end{align*}
where we have used $\alpha_t\geq \frac{1}{4}(t+1)$ and therefore $\sum_{t=0}^{T-1}\alpha_t\geq  T^2/32$.
 
\end{proof}

\subsection{Proof of Lemma~\ref{lem:LemGen}}
\label{app:Prooflem:LemGen}
\begin{proof}
Our starting point is bounding $\alpha_t (f(x_{t+1})- f(z))$ which can be decomposed as follows,
\begin{align} \label{eq:XtBound}
\alpha_t (f(x_{t+1})- f(z))
&\leq
\alpha_t g_t\cdot(x_{t+1}-z)  \nonumber\\
&=
  \alpha_t g_t\cdot(z_{t}-z)  + \alpha_t g_t\cdot(x_{t+1}-z_{t})
\end{align} 
where we use $g_t = \nabla f(x_{t+1})$ in conjunction with the gradient inequality.
Let us now bound the terms in the above equation.
 
\paragraph{(a) Bounding $\alpha_t g_t\cdot(z_{t}-z)$:} 
 The next lemma enables to bound  this term,
 \begin{lemma}\label{lemma:PGDconstrained}
The following holds,
\begin{align*}
\alpha_t g_t\cdot(z_t-z) \leq 
\left(\alpha_t g_t\cdot (z_t-z_{t+1}) -\frac{1}{2\eta_t}\|z_t-z_{t+1}\|^2\right)+ \frac{1}{2\eta_t}\left( \|z_t-z\|^2 -\|z_{t+1}-z\|^2 \right)
\end{align*}
\end{lemma}
The proof of Lemma~\ref{lemma:PGDconstrained} is provided in Appendix~\ref{app:Prooflemma:PGDconstrained}.

We can now relate the first term in the above lemma to $y_{t+1}$.
Define $v = \tau_t z_{t+1} + (1-\tau_t)y_t \in\K$, and notice that $x_{t+1}-v = \tau_t(z_{t}-z_{t+1})$. Using this we may write,
\begin{align}\label{eq:RelationZY}
\alpha_t g_t \cdot (z_t-z_{t+1})& -\frac{1}{2\eta_t}\|z_t-z_{t+1}\|^2   \nonumber\\
&=
\frac{\alpha_{t}}{\tau_t}g_t \cdot (x_{t+1}-v) - \frac{1}{2\eta_t \tau_t^2}\|x_{t+1}-v\|^2  \nonumber\\
&=
\alpha_t^2\left( g_t\cdot(x_{t+1}-v) - \frac{1}{2\eta_t}\|x_{t+1}-v \|^2  \right)  \nonumber\\
&=
\alpha_t^2 g_t\cdot x_{t+1}-\alpha_t^2\left( g_t \cdot v + \frac{1}{2\eta_t}\|x_{t+1}-v \|^2  \right) \nonumber\\
&\le
\alpha_t^2 g_t\cdot x_{t+1}-\alpha_t^2\left( g_t \cdot y_{t+1} + \frac{1}{2\eta_t}\|x_{t+1}-y_{t+1} \|^2  \right)  \nonumber\\
&=
\alpha_t^2 g_t\cdot (x_{t+1} - y_{t+1}) - \frac{\alpha_t^2}{2\eta_t}\|x_{t+1}-y_{t+1} \|^2
\end{align}
where we use $\tau_t =1/\alpha_t$; also in the inequality  we use the following equivalent form for the update rule of $y_{t+1}$,
\begin{align*} 
y_{t+1}= \argmin_{x\in\reals^d}g_t \cdot x + \frac{1}{2\eta_t}\|x-x_{t+1} \|^2~.
\end{align*}
this equivalence can be directly validated  by finding the global optimum of the above objective and showing that it is obtained by choosing $y_{t+1} = x_{t+1}-\eta_t g_t$.

Combining Eq.~\eqref{eq:RelationZY} with Lemma~\ref{lemma:PGDconstrained} gives,
\begin{align}\label{eq:ClassGrad}
\alpha_t g_t\cdot(z_t-z) \leq 
\alpha_t^2 g_t\cdot (x_{t+1} - y_{t+1}) - \frac{\alpha_t^2}{2\eta_t}\|x_{t+1}-y_{t+1} \|^2+ \frac{1}{2\eta_t}\left( \|z_t-z\|^2 -\|z_{t+1}-z\|^2 \right)
\end{align}

\paragraph{(b) Bounding $\alpha_t g_t\cdot(x_{t+1}-z_t)$:} 
 Notice that re-arranging the relation between $x_{t+1}, y_t, z_t$ (recall $x_{t+1} =  \tau_t z_t + (1-\tau_t) y_t$) gives,
\begin{align} \label{eq:relAverage}
x_{t+1}-z_t = r_t ( y_{t}-x_{t+1})
\end{align}
where we denote $r_t =  {(1-\tau_t)}/{\tau_t}$.
Also note that the smoothness of $f$ implies,
\begin{align} \label{eq:Smoothness}
f(y_{t+1}) - f(x_{t+1}) &\leq g_t \cdot(y_{t+1} - x_{t+1}) +\frac{\beta}{2} \| y_{t+1} - x_{t+1}\|^2 
\end{align}
Combining Eq.~\eqref{eq:relAverage} and \eqref{eq:Smoothness} we get,
\begin{align} \label{eq:XZrelation}
g_t &\cdot (x_{t+1}-z_t) \nonumber \\
&= 
r_t \nabla f(x_{t+1}) \cdot (y_{t}-x_{t+1}) \nonumber \\
&\leq
r_t \left( f(y_t) - f(x_{t+1}) \right) \nonumber \\
& = 
r_t \left( f(y_{t}) - f(y_{t+1}) \right) +(r_t+1) \left( f(y_{t+1}) - f(x_{t+1}) \right) -  \left( f(y_{t+1}) - f(x_{t+1}) \right)  \nonumber \\
&\leq
(\alpha_t-1) \left( f(y_{t}) - f(y_{t+1}) \right) 
+
\alpha_t \left(
 g_t \cdot(y_{t+1} - x_{t+1}) +\frac{\beta}{2} \| y_{t+1} - x_{t+1}\|^2  \right)  \nonumber \\
&\quad -  \left( f(y_{t+1}) - f(x_{t+1}) \right)
\end{align}
where second line uses the gradient inequality. We have also used $r_t = (1-\tau_t)/{\tau_t} = \alpha_t-1$ (see Alg.~\ref{algorithm:UniAccel}).

\paragraph{(c) Bounding $\alpha_t\cdot(f(y_{t+1})-f(z))$:}  
 Combining Equations~\eqref{eq:XtBound},~\eqref{eq:ClassGrad} and $\eqref{eq:XZrelation}$ we get,
\begin{align*}
\alpha_t &(f(x_{t+1})- f(z)) \\
& \leq
\alpha_t g_t \cdot (z_{t}-z)  + \alpha_t g_t \cdot (x_{t+1}-z_t) \\
&\leq
\left\{ 
\alpha_t^2 g_t\cdot (x_{t+1} - y_{t+1}) - \frac{\alpha_t^2}{2\eta_t}\|x_{t+1}-y_{t+1} \|^2+ \frac{1}{2\eta_t}\left( \|z_t-z\|^2 -\|z_{t+1}-z\|^2 \right)
\right\} 
\\
\qquad
&\quad+
(\alpha_t^2-\alpha_t) \left( f(y_{t}) - f(y_{t+1}) \right) 
+
\alpha_t^2\left( g_t \cdot(y_{t+1} - x_{t+1}) +\frac{\beta}{2} \| y_{t+1} - x_{t+1}\|^2  \right) \\
&\quad        -\alpha_t\left( f(y_{t+1}) - f(x_{t+1}) \right)
  \\
 &=
 (\alpha_t^2-\alpha_t) ( f(y_{t}) - f(y_{t+1}) )
 + \frac{\alpha_t^2}{2}\left( \beta - \frac{1}{\eta_t} \right)\| y_{t+1} - x_{t+1}\|^2 \\
&\quad    + \frac{1}{2\eta_t}\left( \|z_t-z\|^2 -\|z_{t+1}-z\|^2 \right)      -\alpha_t\left( f(y_{t+1}) - f(x_{t+1}) \right)
\end{align*} 

Re-arranging the above equation implies,
\begin{align*}
\alpha_t &(f(y_{t+1})- f(z)) \\
& \leq
 (\alpha_t^2-\alpha_t) ( f(y_{t}) - f(y_{t+1}) )
 + \frac{\alpha_t^2}{2}\left( \beta - \frac{1}{\eta_t} \right)\| y_{t+1} - x_{t+1}\|^2 \\
&\quad    + \frac{1}{2\eta_t}\left( \|z_t-z\|^2 -\|z_{t+1}-z\|^2 \right)   
\end{align*} 
which concludes the proof.
\end{proof}

\subsection{Proof of Lemma~\ref{lemma:PGDconstrained}} \label{app:Prooflemma:PGDconstrained}
\begin{proof}
Writing the update of the $z_t$'s explicitly we have, 
$$z_{t+1} \gets \argmin_{x\in\K}\| x - \left(z_t - \eta_t\alpha_t g_t \right)\|^2~.$$
Simplifying the above implies the following equivalent form,
$$
z_{t+1}\gets \argmin_{x\in\K} \alpha_t g_t\cdot x +\frac{1}{\eta_t}\R_{z_t}(x)~,
$$
where $\R_{z_t}(x) : = \|x-z_t\|^2/2$.
Since $z_{t+1}$ is a solution of the above minimization problem it satisfies the first order optimality conditions, i.e. $\forall z\in\K$,
\begin{align}\label{e:OptZ_t}
\alpha_t g_t\cdot(z-z_{t+1}) + \frac{1}{\eta_t}\nabla\R_{z_t}(z_{t+1})\cdot (z-z_{t+1})\ge 0
\end{align}
which follows by the first order optimality conditions for $z_{t+1}$.
We are now ready to complete the proof,
\begin{align*}
\alpha_t g_t\cdot(z_t-z)
&=
\alpha_t g_t\cdot (z_t-z_{t+1}) + \alpha_t g_t \cdot (z_{t+1}-z) \\
&\le
\alpha_t g_t\cdot (z_t-z_{t+1})  -\frac{1}{\eta_t}\nabla \R_{z_t}(z_{t+1}) \cdot (z_{t+1}-z) \\
&=
\alpha_t g_t\cdot (z_t-z_{t+1}) -\frac{1}{2\eta_t}\|z_t-z_{t+1}\|^2+ \frac{1}{2\eta_t}\left( \|z_t-z\|^2 -\|z_{t+1}-z\|^2 \right)
\end{align*}
where the second line follows due to Eq.~\eref{e:OptZ_t}, and the second line is due to following lemma (which may be easily extended to general Bergman divergences),
\begin{lemma}\label{lem:Triangle}
Let $u,v,z\in\reals^d$, and let $\R_v(x): = \frac{1}{2}\| x-v\|^2$, then
$$
-\nabla \R_v(u)\cdot (u-z) = \frac{1}{2}\|v-z \|^2 -\frac{1}{2}\|u-z\|^2 - \frac{1}{2}\|u-v \|^2 
$$
\end{lemma}
Below we provide the proof of this lemma.

\end{proof}

\subsubsection{Proof of Lemma~\ref{lem:Triangle}}
\begin{proof}
Noticing that $-\nabla \R_v(u) = v-u$ the lemma may be validated by a direct calculation.
Indeed, $-\nabla \R_v(u)\cdot (u-z)  = -v\cdot z + u\cdot z + u\cdot v -\|u\|^2$. Also,
$$
 \|v-z \|^2 -\|u-z\|^2 - \|u-v \|^2  = -2v\cdot z + 2u\cdot z + 2u\cdot v -2\|u\|^2
$$
\end{proof}

\subsection{Proof of Lemma~\ref{lem:Alphas1}}
\begin{proof}
For $t\leq 3$ we have  $\alpha_t^2-\alpha_t=0$ and the lemma immediately follows.
For $t>3$ we have,
\begin{align*}
(\alpha_t^2-\alpha_t) - (\alpha_{t-1}^2-\alpha_{t-1}) 
&= \frac{(t+1)^2 - 4(t+1)}{16} - \frac{t^2 - 4t}{16} =\frac{2t-3}{8}\leq \alpha_{t-1}/2
\end{align*}
\end{proof}
\label{app:Prooflem:Alphas1}

\subsection{Proof of Lemma~\ref{lem:SqrtSumReversed}}
\label{app:Prooflem:SqrtSumReversed}
\begin{proof}
\textbf{First direction:}
We will prove this part  by induction. The base case, $n=1$, immediately holds.
For the induction step assume that the lemma holds for $n-1$ and let us show it holds for $n$.
By the induction assumption,
\begin{align*}
\sum_{i=1}^n \frac{a_i}{\sqrt{\sum_{j=1}^i a_j}} 
&\geq
\sqrt{\sum_{i=1}^{n-1}a_i} + \frac{a_n}{\sqrt{\sum_{i=1}^n a_i}} = \sqrt{Z-x} + \frac{x}{\sqrt{Z}}
\end{align*}
where we denote $x:= a_n$ and $Z = \sum_{i=1}^n a_i$ (note that $x\leq Z$).
Thus, in order to prove the lemma it is sufficient to show that,
$$
\sqrt{Z-x} + \frac{x}{\sqrt{Z}} \geq \sqrt{Z}~,
$$
which we do next. 
Multiplying both sides by $\sqrt{Z}$ we get that the above is equivalent to,

$$
\sqrt{Z^2-xZ} \ge Z-x
$$
Taking the square of the above an re-ordering we get that the above is equivalent to,
$$
x\le Z
$$
Which holds in our case since $x = a_n\le \sum_{i=1}^n a_i =Z$. This concludes the first part of the proof.

\textbf{Second direction:}
The second inequality in the lemma is due to Lemma $7$ in \citep{mcmahan2010adaptive}.
For completeness we include their proof.

This part is also proved  by induction. The base case, $n=1$, immediately holds.
For the induction step assume that the lemma holds for $n-1$ and let us show it holds for $n$.
By the induction assumption,
\begin{align*}
\sum_{i=1}^n \frac{a_i}{\sqrt{\sum_{j=1}^i a_j}} 
&\leq
2\sqrt{\sum_{i=1}^{n-1}a_i} + \frac{a_n}{\sqrt{\sum_{i=1}^n a_i}} = 2\sqrt{Z-x} + \frac{x}{\sqrt{Z}}
\end{align*}
where we denote $x:= a_n$ and $Z = \sum_{i=1}^n a_i$ (note that $x\leq Z$).
The derivative of the right hand side with respect to $x$ is $-\frac{1}{\sqrt{Z-x}}+\frac{1}{\sqrt{Z}}$
, which is negative for $x\geq 0$. Thus, subject to the constraint $x\geq 0$, the right hand side is maximized at $x=0$, and is therefore at most $2\sqrt{Z}$. This concludes the second part of the proof.
\end{proof}

\subsection{Proof of Lemma~\ref{lem:Log_sum}}
\label{app:Prooflem:Log_sum}
\begin{proof}
We will prove the statement by induction over $n$. 
The base case $n=1$  holds since,
$$
\frac{a_1}{1+a_1} \leq 1 \leq 1+\log(1+a_1)~.
$$
 For the induction step, let us assume that the guarantee holds for $n-1$, which implies that for any $a_1,\ldots, a_n\geq 0$,
\begin{align*}
\sum_{i=1}^{n} \frac{a_i}{1+\sum_{j=1}^i a_j} 
\le 
1+\log(1+ \sum_{i=1}^{n-1} a_i) + \frac{a_n}{1+\sum_{i=1}^n a_i}~.
\end{align*}
The above suggests that establishing following inequality concludes the proof,
\begin{align} \label{eq:Induction}
1+\log( 1+\sum_{i=1}^{n-1} a_i) + \frac{a_n}{1+\sum_{i=1}^n a_i}
\le 
1+\log( 1+\sum_{i=1}^{n} a_i) ~.
\end{align}
Using the notation  $x = a_n/(1+\sum_{i=1}^{n-1}a_i)$, Equation~\eqref{eq:Induction} is equivalent to the following,
\begin{align*} 
 \log(x+1) -\frac{x}{1+x} \ge 0~.
\end{align*}
However, it is immediate to validate that the function $M(x) = \log(x+1) -\frac{x}{1+x}$, is non-negative for any $x\geq 0 $, which establishes the lemma.
\end{proof}

\newpage
\section{Proofs for the General Convex Case (Thm.~\ref{thm:MainNonSmooth})}
\label{app:MainNonSmooth}
Here we provide the complete  proof  of Theorem~\ref{thm:MainNonSmooth}, and of the related lemmas.
For brevity, we will use $z\in\K$ to denote a \emph{global mimimizer} of $f$ which belongs to $\K$. 

Recall that  Algorithm~\ref{algorithm:UniAccel} outputs a weighted average of the queries. Consequently,  we may  employ Jensen's inequality to bound its error as follow,
\begin{align} \label{eq:JensenProof}
f(\bar{y}_T) - f(z)
&\leq 
\frac{1}{\sumtt \alpha_t}\sum_{t=0}^{T-1} {\alpha_t}\left(  f(y_{t+1})-  f(z)\right) ~.
\end{align} 
Combining this with  $\sumtt \alpha_t \geq \Omega(T^2)$, implies that  in order to substantiate the proof it is sufficient to show that, $\sum_{t=0}^{T-1} {\alpha_t}\left(  f(y_{t+1})-  f(z)\right)$, is bounded by $\tO(T^{3/2})$. This is the bulk of the analysis.

We start with the following lemma which provides us with a bound on $\alpha_t \left(  f(y_{t+1})-f(z)\right) $,
\begin{lemma}\label{lem:LemGenNonSmooth}
Assume that $f$ is convex and $G$-Lipschitz. Then for any
 sequence of non-negative weights $\{\alpha_t\}_{t\geq0}$, and learning rates $\{\eta_t\}_{t\geq0}$, Algorithm~\ref{algorithm:UniAccel} ensures the following to hold,
\begin{align*}
\alpha_t &(f(y_{t+1})- f(z)) \\
&\leq
\eta_t \alpha_t^2 \|g_t\|^2+\eta_t \alpha_t^2 \|g_t\|G+ \frac{1}{2\eta_t}\left( \|z_t-z\|^2 -\|z_{t+1}-z\|^2 \right)
+
(\alpha_t^2-\alpha_t) \left( f(y_{t}) - f(y_{t+1}) \right) 
\end{align*} 
\end{lemma}

The proof of Lemma~\ref{lem:LemGenNonSmooth} is provided in Appendix~\ref{app:Prooflem:LemGenNonSmooth}.
We are now ready to prove Theorem~\ref{thm:MainNonSmooth}.

\begin{proof}[Proof  of  Theorem~\ref{thm:MainNonSmooth}]
According to  Lemma~\ref{lem:LemGenNonSmooth},
\begin{align} \label{eq:LemmaBoundNonSmooth}
\sum_{t=0}^{T-1}&\alpha_t (f(y_{t+1})- f(z))  \nonumber\\
& \leq
\underset{\rA}{\underbrace{ \sum_{t=0}^{T-1}\frac{1}{2\eta_t}\left( \|z_t-z\|^2 -\|z_{t+1}-z\|^2 \right)   }}
 +
 \underset{\rB}{\underbrace{ \sum_{t=0}^{T-1}(\alpha_t^2-\alpha_t) ( f(y_{t}) - f(y_{t+1}) ) }}
\nonumber\\
&\quad
 + 
  \underset{\rC}{\underbrace{  \sum_{t=0}^{T-1}\eta_t \alpha_t^2 \|g_t\|^2  }}
  +
    \underset{\rD}{\underbrace{  \sum_{t=0}^{T-1}\eta_t \alpha_t^2 \|g_t\|G  }}
 \end{align}

It is natural to separately bound each of the sums above.

\paragraph{(a) Bounding  $\rA$ :}
Similarly to part $\rm{(a)}$ in the proof of Theorem~\ref{thm:Main} we can show the following to hold,
\begin{align}\label{eq:etaNonSmooth}
\sum_{t=0}^{T-1} \frac{1}{2\eta_t}\left( \|z_t-z\|^2 -\|z_{t+1}-z\|^2 \right) 
&\leq
\frac{D^2}{\eta_{T-1}}
\end{align}

\paragraph{(b) Bounding  $\rB$ :}
Similarly to part $\rm{(b)}$ in the proof of Theorem~\ref{thm:Main} we can show the following to hold for $z=\argmin_{z\in\reals^d}f(x)$,
 \begin{align} \label{eq:xDiffNonSmooth}
\sum_{t=0}^{T-1} (\alpha_t^2-\alpha_t)\left( f(y_{t}) - f(y_{t+1})  \right) 
&\leq
\frac{1}{2}\sum_{t=0}^{T-1} \alpha_t \left(f(y_{t+1})-f(z) \right)
 \end{align}

\paragraph{(c) Bounding  $\rC$ :}
Note that by the definition of $\eta_t$ we have 
$$
\eta_t = \frac{2D}{\left(G^2+\sum_{\tau=1}^t\alpha_\tau^2\|g_\tau\|^2  \right)^{1/2}} 
\le
 \frac{2D}{\left(\sum_{\tau=1}^t\alpha_\tau^2\|g_\tau\|^2  \right)^{1/2}}~.
$$
Using the above ineuality we get,
\begin{align}\label{eq:TermCNonSmooth}
\sum_{t=0}^{T-1}\eta_t \alpha_t^2 \|g_t\|^2 
&\le
2D\sumtt \frac{ \alpha_t^2 \|g_t\|^2 }{\left( \sum_{\tau=0}^t \alpha_\tau^2\|g_\tau\|^2 \right)^{1/2}}
\le
4D\sqrt{\sum_{t=0}^{T-1}  \alpha_t^2  \|g_t\|^2}
 \end{align} 
 where the second inequality uses  Lemma~\ref{lem:SqrtSumReversed}. 

\paragraph{(d) Bounding  $\rD$ :} 
Writing down $\eta_t$ explicitly we get,
\begin{align}\label{eq:TermDNonSmooth}
\sum_{t=0}^{T-1}\eta_t \alpha_t^2 \|g_t\|G
&=
2DG\sumtt \frac{ \alpha_t^2 \|g_t\| }{\left( G^2+\sum_{\tau=0}^t \alpha_\tau^2\|g_\tau\|^2 \right)^{1/2}}  \nonumber\\
&\le
2DGT\sumtt \frac{ \alpha_t \|g_t\| }{\left( G^2+\sum_{\tau=0}^t \alpha_\tau^2\|g_\tau\|^2 \right)^{1/2}}  \nonumber\\
&=
2DGT\sumtt \frac{ \alpha_t (\|g_t\| /G)}{\left( 1+\sum_{\tau=0}^t \alpha_\tau^2(\|g_\tau\|/G)^2 \right)^{1/2}} \nonumber\\
&\leq
10DG\sqrt{\log T}\cdot T^{3/2}~.
 \end{align} 
 where we  used $\forall t\leq T;\; \alpha_t\leq T$. The last line uses the following lemma (see proof in
 Appendix~\ref{app:Prooflem:SqrtSum3}),
\begin{lemma}\label{lem:SqrtSum3}
Consider the $\alpha_t$'s used by our algorithm, i.e., 
\begin{equation}\nonumber
\alpha_t=
\begin{cases}

	1 	&\quad \text{$0\leq t \leq 2$ } \\ 
	\frac{1}{4}(t+1)            	&\quad \text{$t\geq 3$}\\ 

\end{cases}
\end{equation}
And assume a sequence of non-negative numbers, $b_0,b_1,\ldots,b_{T-1}\in[0,1]$. Then the following holds,
$$
\sum_{t=0}^{T-1}\frac{\alpha_t b_t}{\left(1+\sum_{\tau=0}^t \alpha_\tau^2 b_\tau^2\right)^{1/2}} \leq  5\sqrt{\log T}\sqrt{T}
$$
\end{lemma}

\paragraph{Final Bound :}
Combining the bounds on the different terms, Eq.~\eqref{eq:etaNonSmooth}-\eqref{eq:TermDNonSmooth},
 together with Eq.~\eqref{eq:LemmaBoundNonSmooth}, we have,
\begin{align*} 
\sum_{t=0}^{T-1}&\alpha_t (f(y_{t+1})- f(z))  \nonumber\\
& \leq
\frac{D^2}{\eta_{T-1}}
 +
\frac{1}{2}\sum_{t=0}^{T-1} \alpha_t \left(f(y_{t+1})-f(z) \right) \\
&\quad
 + 
4D\sqrt{\sum_{t=0}^{T-1}  \alpha_t^2  \|g_t\|^2}  
+
10DG\sqrt{\log T}\cdot T^{3/2}
 \end{align*} 
Re-arranging and using the explicit expression for $\eta_{T-1}$ we get,
\begin{align*} 
\frac{1}{2}\sum_{t=0}^{T-1}&\alpha_t (f(y_{t+1})- f(z)) \\
& \leq
5D\sqrt{G^2+\sum_{t=0}^{T-1}  \alpha_t^2  \|g_t\|^2}  
+
10DG\sqrt{\log T}\cdot T^{3/2}  \\
&\leq
5DG\sqrt{1+T^3  }  +
10DG\sqrt{\log T}\cdot T^{3/2}  \\
&\leq
20DG\sqrt{\log T}\cdot T^{3/2}~.
 \end{align*} 
 where we have used $\|g_t\|\leq G$, and also, $\alpha_t \leq t+1$ implying that $\sum_{t=0}^{T-1}\alpha_t^2\leq  T^3$.

Using Jensen's inequality we are now ready to establish the final bound,
\begin{align*}
f(\bar{y}_T) - f(z) 
&\leq 
\frac{\sum_{t=0}^{T-1}\alpha_t (f(y_{t+1}) - f(z))}{\sumtt \alpha_t} \\
&\leq
\frac{40\cdot DG\sqrt{\log T}\cdot T^{3/2}}{T^2/32} \\
&=
O\left( DG\sqrt{\log T}/{\sqrt{T} }\right) 
\end{align*}
where we have used $\alpha_t\geq \frac{1}{4}(t+1)$ and therefore $\sum_{t=0}^{T-1}\alpha_t\geq  T^2/32$.

\end{proof}

\subsection{Proof  of  Lemma~\ref{lem:LemGenNonSmooth}}
\label{app:Prooflem:LemGenNonSmooth}
\begin{proof}
Our starting point is bounding $\alpha_t (f(x_{t+1})- f(z))$ which can be decomposed as follows,
\begin{align} \label{eq:XtBoundNonSmooth}
\alpha_t (f(x_{t+1})- f(z))
&\leq
\alpha_t g_t\cdot(x_{t+1}-z)  \nonumber\\
&=
  \alpha_t g_t\cdot(z_{t}-z)  + \alpha_t g_t\cdot(x_{t+1}-z_{t})
\end{align} 
where we use $g_t = \nabla f(x_{t+1})$ in conjunction with the gradient inequality.
Let us now bound the terms in the above equation.
 
\paragraph{(a) Bounding $\alpha_t g_t\cdot(z_{t}-z)$:} 
Similarly to the proof of Lemma~\ref{lem:LemGen} we can show the following to hold (see Eq.~\eqref{eq:ClassGrad} in Lemma~\ref{lem:LemGen}),
\begin{align*}
\alpha_t g_t\cdot(z_t-z) \leq 
\alpha_t^2 g_t\cdot (x_{t+1} - y_{t+1}) - \frac{\alpha_t^2}{2\eta_t}\|x_{t+1}-y_{t+1} \|^2+ \frac{1}{2\eta_t}\left( \|z_t-z\|^2 -\|z_{t+1}-z\|^2 \right)
\end{align*}
Combining the above with $\|x_{t+1} - y_{t+1}\| = \eta_t\|g_t\|$ implies,
\begin{align}\label{eq:ClassGradNonSmooth}
\alpha_t g_t\cdot(z_t-z) \leq 
\eta_t \alpha_t^2 \|g_t\|^2+ \frac{1}{2\eta_t}\left( \|z_t-z\|^2 -\|z_{t+1}-z\|^2 \right)
\end{align}

\paragraph{(b) Bounding $\alpha_t g_t\cdot(x_{t+1}-z_t)$:} 
 Notice that re-arranging the relation between $x_{t+1}, y_t, z_t$ (recall $x_{t+1} =  \tau_t z_t + (1-\tau_t) y_t$) gives,
\begin{align*} 
x_{t+1}-z_t = r_t ( y_{t}-x_{t+1})
\end{align*}
where we denote $r_t =  {(1-\tau_t)}/{\tau_t} $.
Using the above we get,
\begin{align} \label{eq:XZrelationNonSmooth}
g_t &\cdot (x_{t+1}-z_t) \nonumber\\
&= 
r_t \nabla f(x_{t+1}) \cdot (y_{t}-x_{t+1}) \nonumber\\
&\leq
(\alpha_t-1) \left( f(y_t) - f(x_{t+1}) \right) \nonumber\\
&\leq
\alpha_t \left( f(y_{t+1}) - f(x_{t+1})\right) - \left( f(y_{t+1}) - f(x_{t+1})\right)
+(\alpha_t-1) \left( f(y_{t}) - f(y_{t+1})\right)    \nonumber\\
&\leq
\alpha_t G\eta_t\|g_t\|
- \left( f(y_{t+1}) - f(x_{t+1})\right)
+(\alpha_t-1) \left( f(y_{t}) - f(y_{t+1})\right)
\end{align}
where second line uses the gradient inequality,  in the third line we used $r_t = (1-\tau_t)/{\tau_t} = \alpha_t-1$ (see Alg.~\ref{algorithm:UniAccel});
and in the last line we used $|f(y_{t+1}) - f(x_{t+1})| \leq G\|y_{t+1}-x_{t+1}\| \leq G\eta_t\|g_t\|$, which follows by the $G$-Lipschitzness of $f$.

\paragraph{(c) Bounding $\alpha_t\cdot(f(y_{t+1})-f(z))$:}  
 Combining Equations~\eqref{eq:XtBoundNonSmooth},~\eqref{eq:ClassGradNonSmooth}, $\eqref{eq:XZrelationNonSmooth}$ we get,
\begin{align*}
\alpha_t &(f(x_{t+1})- f(z)) \\
& \leq
\alpha_t g_t \cdot (z_{t}-z)  + \alpha_t g_t \cdot (x_{t+1}-z_t) \\
&\leq
\left\{ 
\eta_t \alpha_t^2 \|g_t\|^2+ \frac{1}{2\eta_t}\left( \|z_t-z\|^2 -\|z_{t+1}-z\|^2 \right)
\right\} 
\\
\qquad
&\quad+
(\alpha_t^2-\alpha_t) \left( f(y_{t}) - f(y_{t+1}) \right) 
+
\eta_t \alpha_t^2 \|g_t\|G
- \alpha_t\left( f(y_{t+1}) - f(x_{t+1})\right)
\end{align*} 

Re-arranging the above equation and we get,
\begin{align*}
\alpha_t &(f(y_{t+1})- f(z)) \\
&\leq
\eta_t \alpha_t^2 \|g_t\|^2+\eta_t \alpha_t^2 \|g_t\|G+ \frac{1}{2\eta_t}\left( \|z_t-z\|^2 -\|z_{t+1}-z\|^2 \right)
+
(\alpha_t^2-\alpha_t) \left( f(y_{t}) - f(y_{t+1}) \right) 
\end{align*} 
which concludes the proof.
\end{proof}

\subsection{Proof of Lemma~\ref{lem:SqrtSum3}}
\label{app:Prooflem:SqrtSum3}
\begin{proof}
Let us define the following  time variables,
$$
T_0 = \max \left\{t\in\{0,\ldots,T-1\}: \sum_{\tau=0}^t\alpha_\tau^2 b_\tau^2 \leq 1 \right\}
$$
and for any $k\geq 1$
$$
T_k =\max \left\{t\in\{0,\ldots,T-1\}:   4^{k-1}< \sum_{\tau=0}^t\alpha_\tau^2 b_\tau^2 \leq 4^k  \right\}
$$
By the definition of $T_0$, the following applies,
\begin{align}\label{eq:T0Def}
\sum_{\tau=0}^{T_0} \alpha_\tau b_\tau 
&\leq 
\sqrt{T_0+1}\left(\sum_{\tau=0}^{T_0} \alpha_\tau^2 b_\tau^2 \right)^{1/2}
\leq 
\sqrt{T}~.
\end{align}
where in the first inequality we use $\|u\|_1\leq \sqrt{n}\|u\|_2,\; \forall u\in \reals^n$, in the second inequality we use the definition of $T_0$ together with $T_0\leq T-1$.

For the other time variables we can similarly show the following bounds, i.e., $\forall k\geq 1$,
\begin{align}\label{eq:TkDef}
\sum_{\tau=T_{k-1}+1}^{T_k}\alpha_\tau b_\tau 
&\leq 
\sqrt{T_k-T_{k-1}}\left(\sum_{\tau=T_{k-1}+1}^{T_k} \alpha_\tau^2 b_\tau^2 \right)^{1/2}
\leq 
\sqrt{T_k-T_{k-1}}\cdot 2^k
\end{align}
where in the first inequality we use $\|u\|_1\leq \sqrt{n}\|u\|_2,\; \forall u\in \reals^n$, in the second inequality we use the definition of $T_k$.

Using the definition of the time variables together with Equations~\eqref{eq:T0Def},\eqref{eq:TkDef} we get,
\begin{align*}
\sum_{t=0}^{T-1}&\frac{\alpha_t b_t}{\left(1+\sum_{\tau=0}^t \alpha_\tau^2 b_\tau^2\right)^{1/2}}  \\
&=\sum_{t=0}^{T_0}\frac{\alpha_t b_t}{\left(1+\sum_{\tau=0}^t \alpha_\tau^2 b_\tau^2\right)^{1/2}}
+
\sum_{k\geq 1}\sum_{t=T_{k-1}+1}^{T_k}\frac{\alpha_t b_t}{\left(1+\sum_{\tau=0}^t \alpha_\tau^2 b_\tau^2\right)^{1/2}} \\
&\leq
\sum_{t=0}^{T_0} \alpha_t b_t 
+
\sum_{k\geq 1}\sum_{t=T_{k-1}+1}^{T_k}\frac{\alpha_t b_t}{\left(1+4^{k-1}\right)^{1/2}} \\
&\leq
\sqrt{T}
+
\sum_{k\geq 1}\frac{ 1}{2^{k-1}} \sum_{t=T_{k-1}+1}^{T_k}\alpha_t b_t\\
&\leq
\sqrt{T}
+
2\sum_{k\geq 1}\sqrt{T_k-T_{k-1}}
\end{align*}
where in the third line we use $ \sum_{\tau=0}^t\alpha_\tau^2 b_\tau^2 > 4^{k-1}$ which by definition holds for any $T_{k-1}<t\leq T_{k}$.

Thus, we are left to show that $\sum_{k\geq 1}\sqrt{T_k-T_{k-1}}\leq 2\sqrt{\log T}\sqrt{T}$. To do so, first notice that the maximal value of $k$ is bounded as follows,
\begin{align*}
4^{k_{\max}-1}
&\leq
\sum_{t=0}^{T-1}\alpha_t^2 \\
&\leq
\sum_{t=0}^{T-1} (t+1)^2 \\
&\leq
T^3
\end{align*}
Thus, assuming $T\geq2$ we have $k_{\max}\leq 3\log_2 T $, and therefore,
\begin{align*}
\sum_{k\geq 1}\sqrt{T_k-T_{k-1}} 
&=
\sum_{k= 1}^{k_{\max}}\sqrt{T_k-T_{k-1}}  \\
&\leq
\sqrt{k_{\max}} \left( \sum_{k= 1}^{k_{\max}} (T_k-T_{k-1}) \right)^{1/2} \\
&\leq
\sqrt{3\log T} \left( T - T_0\right)^{1/2} \\
&\leq
\sqrt{3\log T} \sqrt{T}~.
\end{align*}
where we used $\|u\|_1\leq \sqrt{n}\|u\|_2,\; \forall u\in \reals^n$ and also $T_{k_{\max}} =T-1$. This established the lemma.
\end{proof}

\newpage
\section{Proof of Theorem~\ref{thm:MainNonSmoothStoch}}
\label{app:ProofMainNonSmoothStoch}
\begin{proof}
For brevity we will not rehearse all of the details  which are similar to the proof of the offline setting, but rather only emphasize the differences compared to the analysis of Theorem~\ref{thm:MainNonSmooth}.
First note the following  which is analogous to Lemma~\ref{lem:LemGenNonSmooth},
\begin{lemma}\label{lem:LemGenStochastic}
Assume that $f$ is convex and $G$-Lipschitz. 
Assume that we invoke  Algorithm~\ref{algorithm:UniAccel} but provide it with noisy gradient estimates (see Eq.~\eqref{eq:NoisyOracle}) rather then the exact ones.
 Then for any
 sequence of non-negative weights $\{\alpha_t\}_{t\geq0}$, and learning rates $\{\eta_t\}_{t\geq0}$,  the following holds,
\begin{align*}
\alpha_t &(f(y_{t+1})- f(z)) \\
&\leq
\eta_t \alpha_t^2 \|\tg_t\|^2+\eta_t \alpha_t^2 \|\tg_t\|G+ \frac{1}{2\eta_t}\left( \|z_t-z\|^2 -\|z_{t+1}-z\|^2 \right)
+
(\alpha_t^2-\alpha_t) \left( f(y_{t}) - f(y_{t+1}) \right) \\
\qquad
&\quad+
\alpha_t (g_t-\tg_t)\cdot(z_{t}-z)
\end{align*} 
\end{lemma}
We prove this lemma in Appendix~\ref{app:Prooflem:LemGenStochastic}.

Now, focusing on the term $\alpha_t (g_t-\tg_t)\cdot(z_{t}-z)$, the unbaisdness of $\tg_t$ immediately implies,
$$
\E[\alpha_t (g_t-\tg_t)\cdot(z_{t}-z)] = 0~.
$$
Ignoring this term and comparing the bound in the above lemma to Lemma~\ref{lem:LemGenNonSmooth}, one can see that the expression are identical up to replacing, $g_t \leftrightarrow \tg_t$.
This identity in the expressions applies also to the learning rate,  $\eta_t$ (again up to replacing, $g_t \leftrightarrow \tg_t$). 
Thus, the exact same analysis as of Lemma~\ref{lem:LemGenNonSmooth} shows that w.p.~$1$ we have,
$$
\sumtt \alpha_t (f(y_{t+1})- f(z)) - \sumtt \alpha_t (g_t-\tg_t)\cdot(z_{t}-z) \leq O(GD\sqrt{\log T}\cdot T^{3/2})~.
$$
Taking expectations and using the above in conjunction with the definition of $\bar{y}_T$ and Jensen's inequality concludes the proof.
\end{proof}

\subsection{Proof of Lemma~\ref{lem:LemGenStochastic}}
\label{app:Prooflem:LemGenStochastic}
\begin{proof}
The proof follows similar lines to the proof of Lemmas~\ref{lem:LemGenNonSmooth} and~\ref{lem:LemGen}. Here we will highlight the changes due to the stochastic setting.

Our starting point is bounding $\alpha_t (f(x_{t+1})- f(z))$ which can be decomposed as follows,
\begin{align} \label{eq:XtBoundStoch}
\alpha_t (f(x_{t+1})- f(z))
&\leq
\alpha_t g_t\cdot(x_{t+1}-z)  \nonumber\\
&=
  \alpha_t \tg_t\cdot(z_{t}-z)  + \alpha_t g_t\cdot(x_{t+1}-z_{t})+\alpha_t (g_t-\tg_t)\cdot(z_{t}-z)
\end{align} 
Due to the unbiasedness of $\tg_t$ then the expectation of the last term $\alpha_t (g_t-\tg_t)\cdot(z_{t}-z)$ is zero.
Let us now bound the remaining two terms in the above equation.
 
\paragraph{(a) Bounding $\alpha_t \tg_t\cdot(z_{t}-z)$:} 
Similarly to the proof of Lemma~\ref{lem:LemGen} we can show the following to hold (see Eq.~\eqref{eq:ClassGrad} in Lemma~\ref{lem:LemGen}),
\begin{align*}
\alpha_t \tg_t\cdot(z_t-z) \leq 
\alpha_t^2 \tg_t\cdot (x_{t+1} - y_{t+1}) - \frac{\alpha_t^2}{2\eta_t}\|x_{t+1}-y_{t+1} \|^2+ \frac{1}{2\eta_t}\left( \|z_t-z\|^2 -\|z_{t+1}-z\|^2 \right)
\end{align*}
Combining the above with $\|x_{t+1} - y_{t+1}\| = \eta_t\|\tg_t\|$ implies,
\begin{align}\label{eq:ClassGradNonSmoothStoc}
\alpha_t \tg_t\cdot(z_t-z) \leq 
\eta_t \alpha_t^2 \|\tg_t\|^2+ \frac{1}{2\eta_t}\left( \|z_t-z\|^2 -\|z_{t+1}-z\|^2 \right)
\end{align}

\paragraph{(b) Bounding $\alpha_t g_t\cdot(x_{t+1}-z_t)$:} 
Similarly to the proof of  Lemma~\ref{lem:LemGenNonSmooth} we can show the following to hold 
(see Eq.~\eqref{eq:XZrelationNonSmooth} therein),
\begin{align} \label{eq:XZrelationNonSmoothStoc}
g_t \cdot (x_{t+1}-z_t)
&\leq
\alpha_t G\eta_t\|\tg_t\|
- \left( f(y_{t+1}) - f(x_{t+1})\right)
+(\alpha_t-1) \left( f(y_{t}) - f(y_{t+1})\right)
\end{align}

\paragraph{(c) Bounding $\alpha_t\cdot(f(y_{t+1})-f(z))$:}  
 Combining Equations~\eqref{eq:ClassGradNonSmoothStoc},~\eqref{eq:XZrelationNonSmoothStoc} and $\eqref{eq:XtBoundStoch}$ we get,
\begin{align*}
\alpha_t &(f(x_{t+1})- f(z)) \\
&\leq
\left\{ 
\eta_t \alpha_t^2 \|\tg_t\|^2+ \frac{1}{2\eta_t}\left( \|z_t-z\|^2 -\|z_{t+1}-z\|^2 \right)
\right\} +\alpha_t (g_t-\tg_t)\cdot(z_{t}-z)
\\
\qquad
&\quad+
(\alpha_t^2-\alpha_t) \left( f(y_{t}) - f(y_{t+1}) \right) 
+
\eta_t \alpha_t^2 \|\tg_t\|G
- \alpha_t\left( f(y_{t+1}) - f(x_{t+1})\right)
\end{align*} 

Re-arranging the above equation and we get,
\begin{align*}
\alpha_t &(f(y_{t+1})- f(z)) \\
&\leq
\eta_t \alpha_t^2 \|\tg_t\|^2+\eta_t \alpha_t^2 \|\tg_t\|G+ \frac{1}{2\eta_t}\left( \|z_t-z\|^2 -\|z_{t+1}-z\|^2 \right)
+
(\alpha_t^2-\alpha_t) \left( f(y_{t}) - f(y_{t+1}) \right) \\
\qquad
&\quad+
\alpha_t (g_t-\tg_t)\cdot(z_{t}-z)
\end{align*} 
which concludes the proof.

\end{proof}

\newpage

\section{Proof of Lemma~\ref{lemma:GradIneqSmooth}}
\label{app:ProofLemmaGradIneqSmooth}
\begin{proof}
The $\beta$ smoothness of $F$ means the following to hold $\forall x,u\in\reals^d$,
$$F(x+u) \leq F(x) +\nabla F(x)^\top u+\frac{\beta}{2}\|u\|^2 ~.$$
Taking  $u=-\frac{1}{\beta}\nabla F(x)$ we get,
$$F(x+u) \le F(x) -\frac{1}{\beta}\|\nabla F(x)\|^2+\frac{1}{2\beta}\|\nabla F(x)\|^2~.$$
Thus:
\begin{align*}
\|\nabla F(x)\| &\le \sqrt{2\beta \big( F(x) -F(x+u)\big)}\\
&  \le  \sqrt{2\beta \big(F(x) -F(x^*)\big)}~,
\end{align*}
where in the last inequality we used $F(x^*) \leq F(x+u)$ which holds since $x^*$ is the \emph{global} minimum.
\end{proof}

\newpage
\section{Additional Numerical Experiments}
\label{app:Numerics}
Here, we present numerical experiments on the stochastic setting, and on a practical variant that neglects the projection steps. 

\subsection{Numerical Experiments on the Stochastic Setting}
We consider the same problem setup as in Section~\ref{sec:Exps}. 
Rather than using the exact gradients, we compute the unbiased estimates evaluated by a single data point (i.e. minibatch of size $1$) 
The results are shown in Figure~\ref{fig:stochastic}.

\begin{figure}[h]
\includegraphics[width=0.70\textwidth]{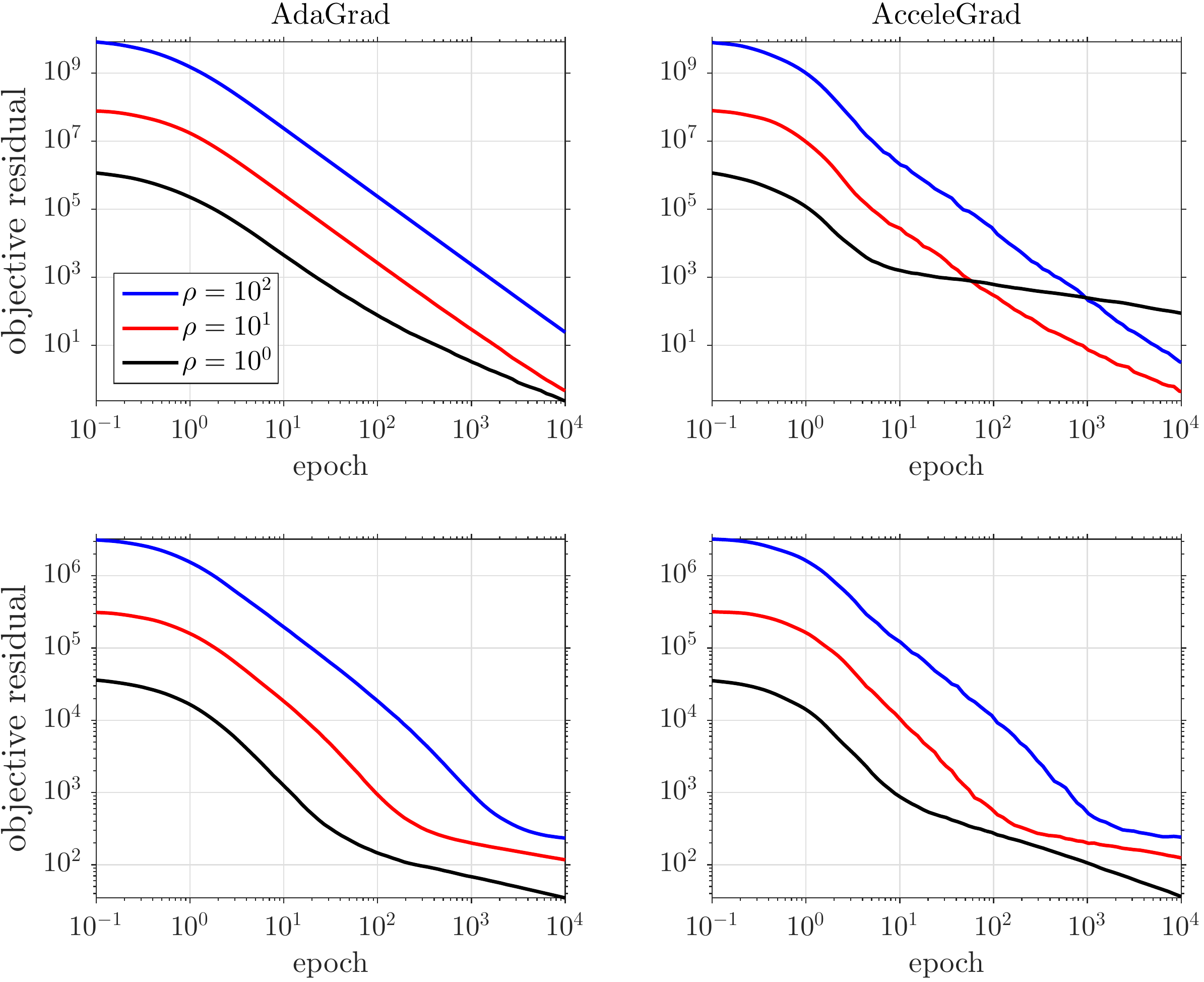}
\centering
\caption{Comparison of AdaGrad and AcceleGrad in stochastic setting for smooth \textit{(top)} and non-smooth \textit{(bottom)} problems. Epoch denotes one full data pass, hence $500$ iterations.}
\label{fig:stochastic}
\end{figure}

AdaGrad and AcceleGrad perform similar empirically for most of the parameter choices. 
AdaGrad overperforms AcceleGrad only for the smooth problem with $\rho = 1$. 
This bahavior is caused by the projection step, and slightly increasing $D$ cures the problem for AcceleGrad. 

Universal gradient methods \citep{nesterov2015universal} are based on a line-search technique that relies on the exact first order oracle information.
Thus, it is not so surprising that in practice these methods fail upon receiving stochastic feedback, and we therefore do not present their performance.

\subsection{Numerical Experiments Neglecting the Projections}

We observed that the methods work well in practice even if we ignore the projection step in the unconstrained setting. 
In some cases, this simple tweak may even improve the performance. 
We used the same test setup as in Section~\ref{sec:Exps}, and the results are shown in Figures~\ref{fig:l1l2min-wo-proj} and \ref{fig:stochastic-wo-proj} for the deterministic and stochastic settings respectively. 
Note that the method works also when we underestimate $D$. 

\begin{figure}[ht!]
\includegraphics[width=\textwidth]{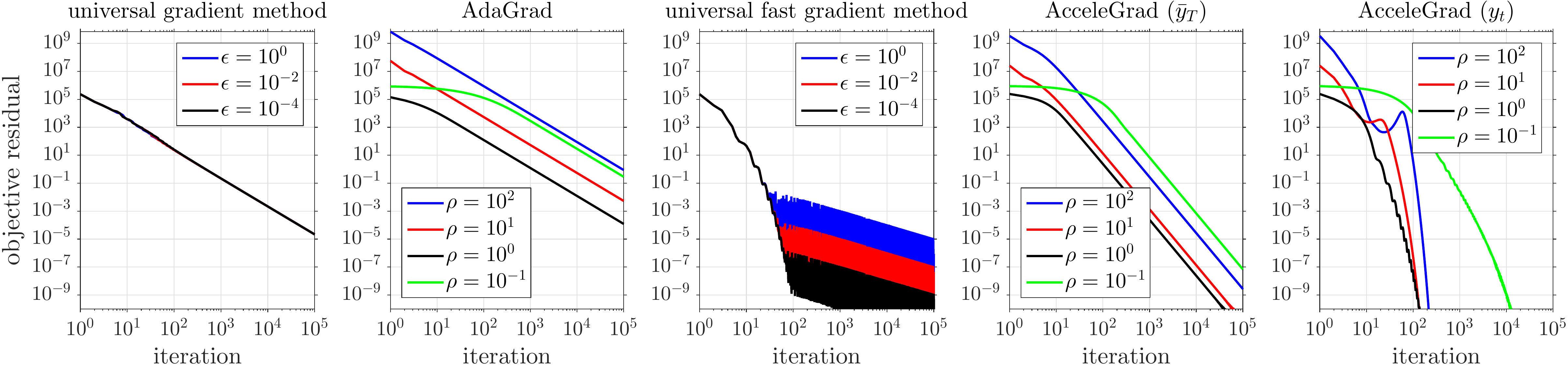} \\ 
\\
\includegraphics[width=\textwidth]{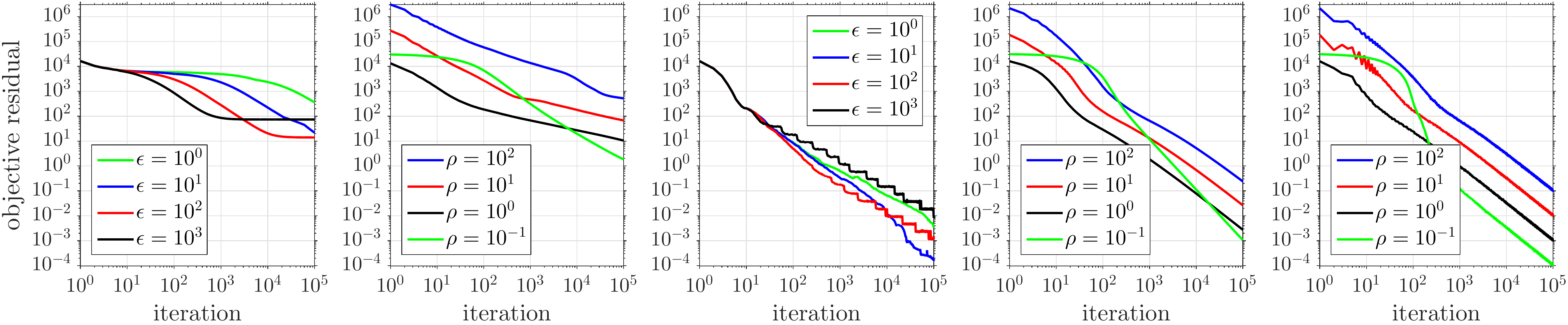}
\centering
\vspace{-3mm}
\caption{Comparison of universal methods at a smooth \textit{(top)} and a non-smooth \textit{(bottom)} problem. Adaptive methods are tweaked to ignore the projection.}
\label{fig:l1l2min-wo-proj}
\end{figure}

\begin{figure}[h]
\includegraphics[width=0.60\textwidth]{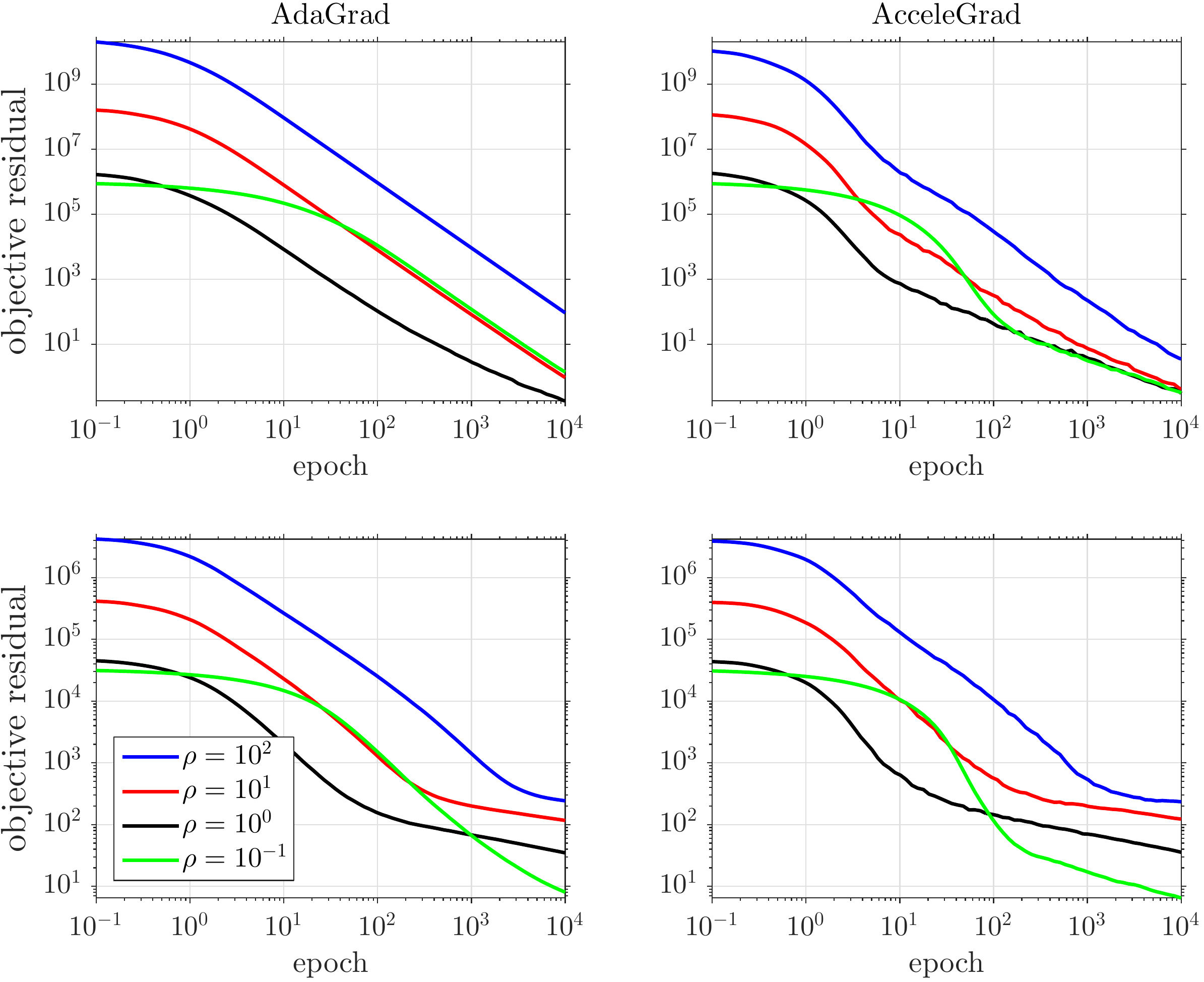}
\centering
\caption{Comparison of AdaGrad and AcceleGrad in stochastic setting for smooth \textit{(top)} and non-smooth \textit{(bottom)} problems. Methods are tweaked to ignore the projection. Epoch denotes one full data pass, hence $500$ iterations.}
\label{fig:stochastic-wo-proj}
\end{figure}

\newpage
\subsection{Experiments with Large Minibatch}
\label{app:ExpLargeBatch}
In this section we apply AcceleGrad to a real world stochastic optimization problem
and compare its performance with AdaGrad.
We examine the effect of  minibatch size verses performance.
The large minibatch regime is important when one likes to apply SGD using several machines in parallel. This is done by dividing the minibatch computation between the machines. 
Unfortunately, it is well known that the performance of SGD degrades with the increase of minibatch size $b$. Here, we show that  AcceleGrad might be more appropriate in this case.

Concretely we consider the RCV1\footnote{available in the UCI repository website (\url{https://archive.ics.uci.edu/ml/})} dataset which is a binary labeled set with $20424$ datapoints samples and  $47366$ features.  
We train a classifier for this dataset using logistic loss (smooth case) as well as using the hinge loss (SVM). 
We compare the performance of  AcceleGrad  with AdaGrad. For each method we examine several minibatch sizes, and observe the performance of each method verses the number of epochs (total number of gradients that we have computed).

The results for logistic regression appear in Figure~\ref{fig:stochastic_minibtch_smooth}.
For AdaGrad we see that the performance degrades as we increase the minibatch size beyond $b=1000$. This actually agrees with theory that predicts a degradation with the increase of $b$.

For AcceleGrad we observe an interesting phenomenon: 
if we aim for a very small error (in this case smaller than $10^{-2}$)
then as we increase the minibatch size  the performance actually improves.  
The intuition behind this is the following: upon using small $b$ the gradients are noisy and both AcceleGrad and AdaGrad will obtain the slow $\O(1/\sqrt{T})$ rate, where $T$ is the number of iterations. However, as  $b$ increases the gradients are becoming more accurate and AcceleGrad with obtain a rate approaching 
$\O(1/T^2)$ while AdaGrad will approach $\O(1/T)$ rate. 
Now note that the number of gradient calculations $S$, depends on $b$ and $T$ as follows, 
$
T = S/b~.
$\\
Thus, for small minibatch, both methods will ensure a rate of  $\O(\sqrt{b}/\sqrt{S})$, which clearly degrades with $b$. As $b$ increases AcceleGrad will obtain a rate approaching 
$\O(b^2/S^2)$ while AdaGrad will approach $\O(b/S)$ rate.

We have observed similar behaviour when train an SVM (i.e., using hinge loss). This can be seen in Figure~\ref{fig:stochastic_minibtch_svm}.

Note that we have performed several other experiments with different $D$ parameters, and also different $\ell
_3$ regularization parameters. In all experiments we have seen the same qualitative behaviour that we describe above.

\begin{figure}[t]
\includegraphics[width=0.60\textwidth]{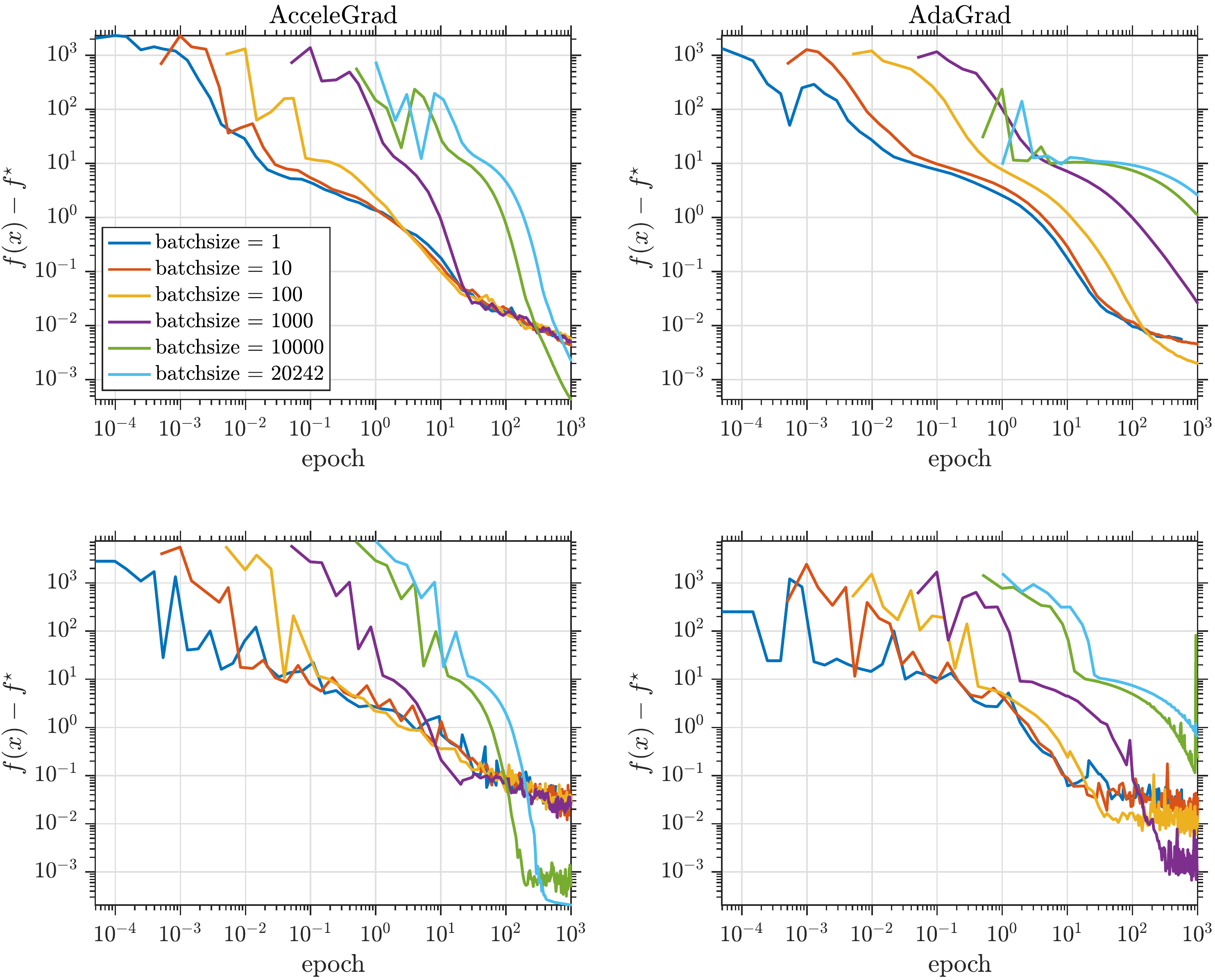}
\centering
\caption{Comparison of AdaGrad and AcceleGrad for logistic regression task using different minibatch sizes. We display the averaged iterates, $\bar{y}_T$ \textit{(top)},  as well as the  non-averaged iterates, $y_t$ \textit{(bottom)}. Both methods use the same parameter $D=10^4$.}
\label{fig:stochastic_minibtch_svm}
\end{figure}

\begin{figure}[h]
\includegraphics[width=0.60\textwidth]{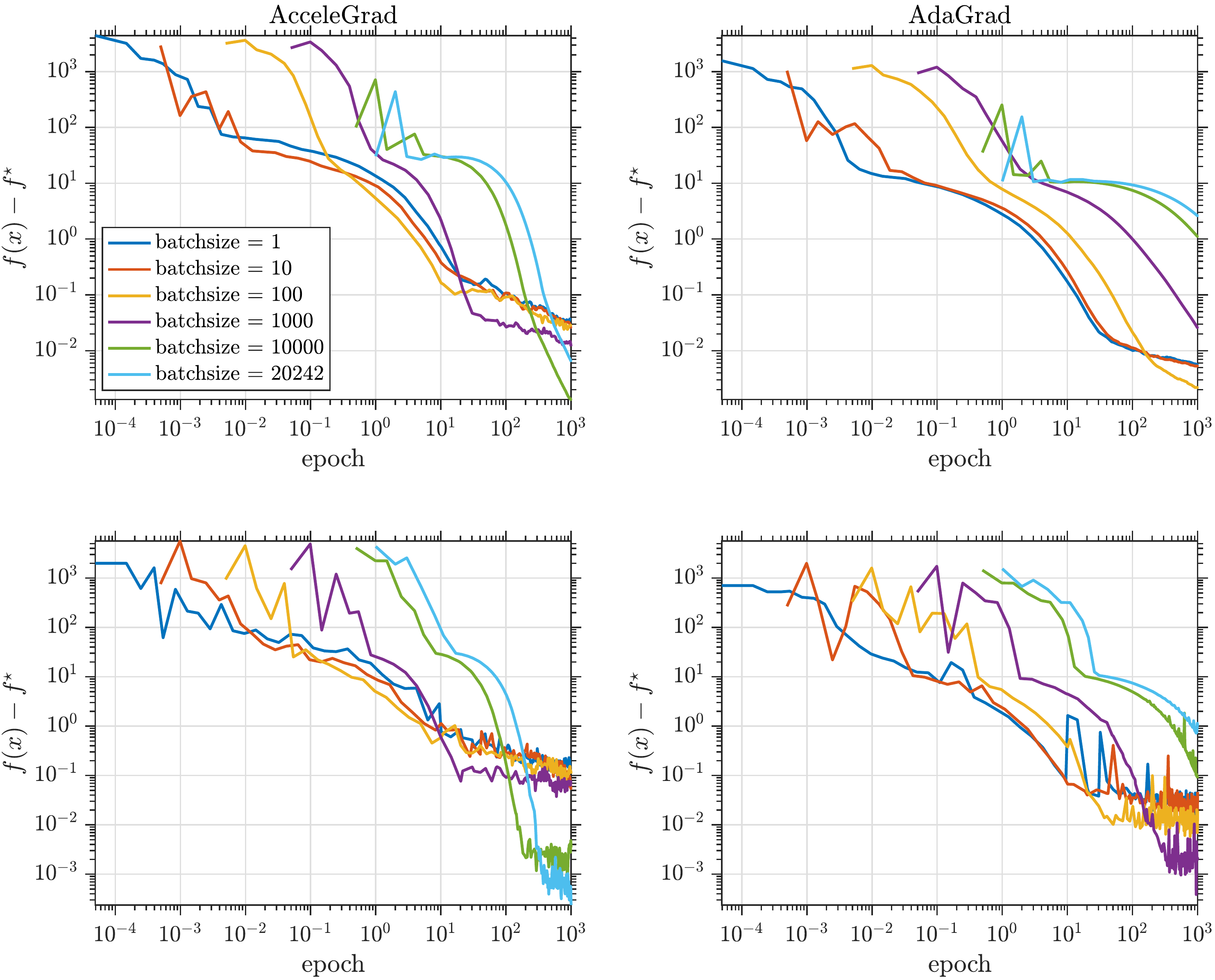}
\centering
\caption{Comparison of AdaGrad and AcceleGrad for training SVM using different minibatch sizes. We display the averaged iterates, $\bar{y}_T$ \textit{(top)},  as well as the  non-averaged iterates, $y_t$ \textit{(bottom)}. Both methods use the same parameter $D=10^4$.}
\label{fig:stochastic_minibtch_smooth}
\end{figure}

\end{document}